\documentclass{article}

\usepackage[final,main]{neurips_2026}

\usepackage{natbib}
\usepackage{algorithmic}
\usepackage{multirow}
\usepackage{makecell}

\usepackage{math_commands}
\usepackage{header}
\usepackage{todonotes}

\usetikzlibrary{arrows.meta}

\newcommand{\Psim}{P_{\text{sim}}}
\newcommand{\PL}{P_L}
\newcommand{\pieval}{\pi_{e}}

\newcommand{\zsim}{\hat{z}}

\title{Controllable User Simulation}

\author{Guy Tennenholtz$^\dagger$\thanks{Correspondence to: \texttt{guytenn@gmail.com}, work done at Google.},~~Ofer Meshi$^\dagger$,~~Amir Globerson$^{\dagger\ddagger}$ \\ \bf Uri Shalit$^{\dagger\ddagger}$,~~Jihwan Jeong$^\ddagger$~~Craig Boutilier$^\dagger$
\\~\\
\large $\dagger$ Google Research, $\ddagger$ Tel Aviv University}

\begin{document}

\maketitle

\begin{abstract}
Using offline datasets to evaluate conversational agents often fails to cover rare scenarios or to support testing new policies. This has motivated the use of \emph{controllable user simulators} for targeted, counterfactual evaluation, typically implemented by prompting or fine-tuning large language models. In this work, we formalize controllable simulation as a causal inference problem. By bridging natural language evaluation with off-policy evaluation methodology, we show that the standard practice of training simulators via supervised fine-tuning on post-hoc trajectory labels yields a structurally biased model. Specifically, these labels are inextricably coupled to the data-generating behavior policy, injecting a \emph{look-ahead bias} that breaks causal consistency. Furthermore, we prove that under policy shift this failure causes the variance of evaluation metrics to explode geometrically, a phenomenon we term \emph{controllability collapse}. To restore causal consistency, we establish theoretical conditions for accurate simulation and propose practical training mitigations: \textit{a priori} controls, step-wise dynamic controls, and direct policy-conditioned learning. Empirical evaluation confirms that while standard global controls distort conversational distributions and collapse behavioral diversity, our causally grounded simulators eliminate look-ahead bias, preserve natural variance, and exhibit robust zero-shot generalization to unseen agent behaviors.
\end{abstract}

\section{Introduction}
\label{section: introduction}

Conversational agents must be continuously evaluated for attributes such as quality, safety, and fairness. Given the complexity of human-agent interactions, these experiments are often conducted under controlled conditions to test specific use cases, such as risky behaviors, known failure modes, specific personas, or domain verticals (e.g., travel planning, apparel recommendations). While agents should ideally be evaluated online with live users, such experiments are prohibitively expensive and pose significant safety and reputational risks by exposing real users to unvalidated policies. Thus, there is a critical need for complementary evaluation methods that are scalable and low-risk, yet maintain validity with respect to real human behavior. To address this, the field has increasingly turned to user simulators \citep{davidson2023user,wang-etal-2023-rethinking,laban2026lost,naous2026flipping}, which enable systematic exploration of an agent's adaptability to diverse user behaviors and constraints without invasive live trials \citep{hsu2024minimizing,zhang2024reformulating,qin2025compass}.

To perform targeted assessments, the evaluators of agents often require \emph{controllable user simulators}, in which a \emph{control variable} is used to steer the simulator toward specific personas, constraints, or outcomes. For example, one might want to simulate a user trying to book a flight who forgets their passport number, or test an agent's de-escalation skills against an increasingly frustrated user. We formally define the goal of controllable simulation as successfully generating user responses from the \emph{true conditional distribution} of user behavior given this control variable.
Furthermore, a successful user simulator must maintain this distribution even when it interacts with a \emph{new} agent policy.

Currently, controllable simulators are typically LLMs prompted or trained on real interactions. A common practice is \emph{post-hoc trajectory-conditioned training} \citep{wang2025know,dou2025simulatorarena,liu2026muse,jin2025twice}: first, researchers collect user logs with a deployed agent. Second, an automated annotator analyzes the \emph{complete} trajectory to extract a global label (e.g., ``User successfully booked a flight but forgot their passport''). Finally, a simulator is trained via supervised fine-tuning (SFT) to predict turn-by-turn utterances conditioned on this static label. At inference, evaluators prompt the simulator with the control variable to test new agents. Despite its intuitive appeal, we show this standard practice compromises the goal of sampling from the true conditional distribution. The core issue is that a post-hoc control label is inextricably coupled to the actions of the \emph{specific agent} used to generate the training data. 

Consider a control variable for a ``frustrated user booking a flight.'' In offline logs, user frustration might stem from the training agent providing unhelpful answers. Here, the simulator may learn to associate the future-outcome label ``ends up frustrated'' with angry utterances. Suppose we now evaluate a \emph{new, highly optimized agent} providing perfect answers. Prompting the simulator with a future-outcome label generates frustrated responses \emph{despite} flawless agent actions, because the conditioning label implicitly encodes the data-generating agent's (poor) future choices. Conditioning on future-informed outcomes breaks the natural conversation sequence, introducing severe \emph{look-ahead bias}. Crucially, this exposes a fundamental trade-off between outcome controllability and causal realism: in a valid de-escalation test, an evaluator cannot rigidly guarantee that frustration will emerge regardless of agent quality. Instead, causally sound controllability must specify pre-interaction traits and reactivity rules (e.g., a ``low-patience'' profile), allowing frustration to emerge or dissolve dynamically in response to the agent's actual behavior.

Building on results from the off-policy evaluation (OPE) literature, we further prove that when this type of simulator evaluates an agent policy that differs from the one in the training data, this mismatch causes the variance of counterfactual evaluation metrics to explode geometrically, a structural failure we dub \emph{controllability collapse}. This collapse is a mathematical artifact of the conditioning mechanism itself: even a ``perfect'' LLM text generator suffers from this breakdown simply because the new evaluation policy diverges from the training policy.

To mitigate these problems and achieve causally sound controllable simulation, we must reframe the control mechanism to respect the natural sequence of interactions. We provide theoretical conditions under which controllable simulation remains accurate, and propose three practical train-time mitigations: (1)~\emph{\textit{A~priori} controls:} restricting the conditioning variables strictly to pre-interaction traits that are completely independent of the agent's future actions; (2) \emph{Step-wise dynamic controls:} generating a dynamic control state in a turn-by-turn fashion based \emph{only} on the observable history up to that point; and (3) \emph{Direct policy-conditioned learning:} resolving the policy mismatch by explicitly conditioning the generative simulator on the target agent's specific policy.

We provide empirical evaluations on two multi-turn dialogue datasets to validate our theoretical results. We demonstrate that standard trajectory-conditioned controls severely distort natural conversational distributions and collapse behavioral diversity. By contrast, our causally-grounded mitigations eliminate look-ahead bias, preserve natural variance, and exhibit robust zero-shot generalization to unseen agent behaviors, providing a rigorous foundation for reliable model-based evaluation.

\vspace{-0.4em}
\section{Problem Formulation}
\label{section: problem formulation}

We model the interaction of a user with an agent as a discrete-time stochastic process. Let $\mathcal{U}$ and $\mathcal{A}$ denote the user and agent action spaces; user actions might be queries, critiques or other conversational utterances, while agent actions include responses the agent may offer. A length $T$ \emph{trajectory} is $\tau = (u_1, a_1, \dots, u_T, a_T)$. Following standard formalisms in sequential decision-making and stochastic processes \citep{kallenberg1997foundations,puterman2014markov}, we define the natural filtration $\mathcal{F}_t = \sigma(u_1, a_1, \dots, u_t, a_t)$ as the observable history up to step~$t$, with realizations $h_t \in \mathcal{H}_t$. We also define $\mathcal{F}_{t-} = \mathcal{F}_{t-1} \vee \sigma(u_t)$ as the state occurring immediately before the agent acts. 

We assume the agent acts according to a \emph{policy} $\pi$, where $a_t \sim \pi(a_{t} \mid h_{t-1}, u_{t})$, which depends only on $h_{t-1}$ and $u_t$. The true user dynamics, reflecting the behavior of random users drawn from some population, is $P(u_t \mid h_{t-1})$, and depends only on the history (hence, they are invariant to the agent's future actions). Let $P^\pi(\cdot)$ be the probability measure over trajectory space induced by $\pi$, which admits the standard causal decomposition:
$
P^\pi(\tau)
=
\prod_{t=1}^{T} P(u_{t}\mid h_{t-1})\pi(a_{t}\mid h_{t-1},u_{t}).
$

A \emph{controllable simulator} conditions the generation of user behavior on some control variable $z \in \gZ$. This control variable is used to steer the simulated behavior toward specific user intents, \emph{a priori} traits or demographics, or specific conversational outcomes. Assuming a joint distribution $P(\tau, z)$ including the control of interest, a controllable simulator induces transition kernels $\Psim(u_t \mid h_{t-1}, z)$ that approximate the true conditional measure $P(u_t \mid h_{t-1}, z)$. In practice, the target agent is often evaluated as a black box, so we deploy the simulator by composing its user transitions with the target agent's policy $\pieval$. When simulating a targeted counterfactual trajectory with a novel evaluation policy $\pieval$ given control $z$, the induced distribution is:
\begin{align}
\label{eq: general controllable simulation}
  \Psim^{\pieval}(\tau \mid z)
  =
  \prod_{t=1}^T \Psim(u_{t}\mid h_{t-1}, z)\pieval(a_{t} \mid h_{t-1}, u_t).
\end{align}

The fundamental goal of controllable simulation is to successfully sample from the true conditional user distribution. We can evaluate the overall trajectory fidelity by analyzing the trajectory density ratio $W_T(z) = \frac{P^{\pieval}(\tau \mid z)}{\Psim^{\pieval}(\tau \mid z)}$. Because the trajectory distribution depends on both the user's generative model and the agent's policy responses, $W_T(z)$ captures deviations in both actors. Controlled simulation succeeds globally for an evaluation policy $\pieval$ if $W_T(z) = 1$ for all valid trajectories. 

To isolate the deviation caused specifically by the user simulator's generation at a given step, we define the step-wise \emph{user generative error} as $\rho_t(u_t \mid h_{t-1}, z) = \frac{P^{\pieval}(u_t \mid h_{t-1}, z)}{\Psim(u_t \mid h_{t-1}, z)}$. Assuming absolute continuity ($P_{\text{sim}}(u_t \mid h_{t-1}, z) > 0$ whenever $P^{\pieval}(u_t \mid h_{t-1}, z) > 0$), we can factor the trajectory density ratio into two distinct error sources:

\begin{align*}
W_T(z) 
= 
\prod_{t=1}^T \underbrace{\rho_t(u_t \mid h_{t-1}, z)}_{\text{User Generative Error}} \underbrace{\left( \frac{P^{\pieval}(a_t \mid h_{t-1}, u_t, z)}{\pieval(a_t \mid h_{t-1}, u_t)} \right)}_{\text{Agent Policy Divergence}}
\end{align*}
If the target agent $\pieval$ ignores $z$ \emph{and} $z$ is a non-descendant of $a_t$ (e.g., an \emph{a priori} trait, \Cref{fig: causal graph}), then by conditional independence $P^{\pieval}(a_t \mid h_{t-1}, u_t, z) = \pieval(a_t \mid h_{t-1}, u_t)$, canceling the \emph{agent policy divergence}. Conversely, if $z$ is a future-dependent post-hoc outcome, conditioning opens a backward causal path via collider bias, breaking statistical independence even if the agent never explicitly observes $z$.\footnote{Dropping control $z$ entirely avoids this bias, but defeats the purpose of controllable evaluation.}

\begin{figure}[t!]
    \centering
    \includegraphics[width=0.9\linewidth]{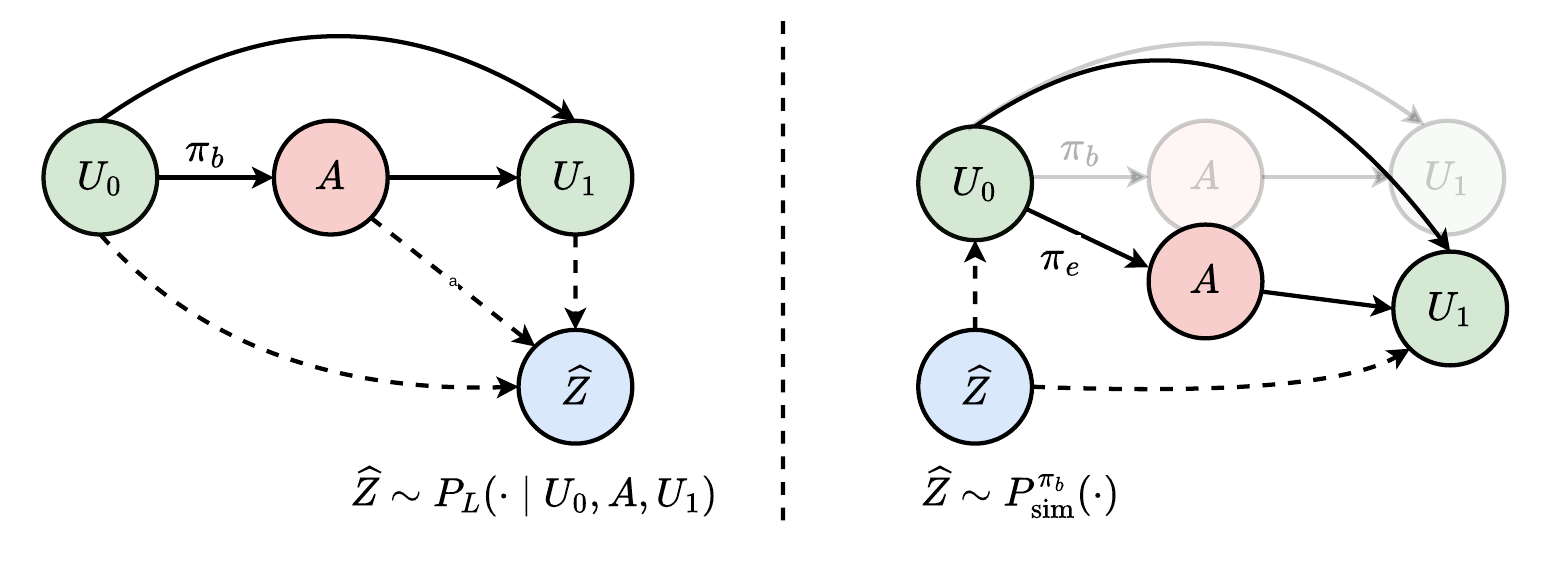}
    \caption{\footnotesize Causal graphs illustrating the look-ahead bias introduced by trajectory-conditioned SFT. \textbf{Left:} Under behavior policy $\pi_b$, control label $\widehat{Z}$ is extracted post-hoc from the trajectory. \textbf{Right:} During evaluation with policy $\pi_e$, conditioning the simulator on $\widehat{Z}$ opens a backward causal path from future outcomes to the current state, breaking the natural filtration.}
    \label{fig: causal graph}
\end{figure}

\section{Post-Hoc Trajectory-Conditioned Training and Its Pitfalls}
\label{section:pitfalls of controllable simulation}

A key question is how to train a simulator from empirical data. Recent work models user simulation as dialogue refactoring, explicitly extracting a static label (e.g., a user profile) from a full historical dialogue $\tau$ to condition turn-by-turn generation \citep{wang2025know,dou2025simulatorarena,liu2026muse,jin2025twice}. However, this approach induces a simulator that is causally biased by design, a flaw that compounds into a complete statistical breakdown when deployed to evaluate novel agent policies.

We formally define \emph{trajectory-conditioned training} as training a user simulator with offline data generated by users engaging with a specific data-gathering agent policy $\pi_b$. Note that this agent policy $\pi_b$ induces a joint distribution $P^{\pi_b}$ over the entire trajectory, including the agent actions, user utterances, and any post-hoc labels. The simulator maximizes the likelihood of user utterances conditioned on a label derived from the full trajectory. Let $\PL(\cdot \mid \tau)$ be a stochastic \emph{post-hoc labeling function} mapping a full trajectory to a distribution over $\gZ$, with $\zsim \sim \PL(\cdot \mid \tau)$ as the sampled control (note that $\PL$ and $\zsim$ are strictly $\mathcal{F}_T$-measurable). This training paradigm constructs learned dynamics converging to $\Psim(u_t \mid h_{t-1}, \zsim) \triangleq P^{\pi_b}(u_t \mid h_{t-1}, \zsim)$. Crucially, this implicitly anchors the simulator to a \emph{learned prior} $\Psim^{\pi_b}(\zsim) = \int \PL(\zsim \mid \tau) dP^{\pi_b}(\tau)$, inextricably coupling the control variable's distribution to the specific actions of the behavior policy $\pi_b$.

\paragraph{Example: Post-Hoc Labeling in Practice.}
Suppose a simulator is trained on support logs where an annotator assigned static traits like $\zsim = \text{``easily frustrated''}$. Crucially, if the behavior policy $\pi_b$ provided unhelpful answers, it actively \emph{caused} this frustration. Thus, the label is not an \emph{intrinsic} user trait, but an artifact of the specific agent's actions.

\paragraph{Policy Dependence of the Control Variable.}
Because a trajectory $\tau$ is inextricably tied to the behavior policy $\pi_b$, unlike a static, $\mathcal{F}_0$-measurable prior, $\zsim$ cannot be disentangled from the agent's actions. Assessing our structural divergence metric $W_T(z)$ exposes how this entanglement breaks the evaluation on two distinct fronts. First, \emph{regardless of the evaluation policy}, conditioning on a post-hoc label creates a backward causal path from future to past actions (see \Cref{fig: causal graph}), injecting a strict look-ahead bias (\Cref{subsection: lookahead bias}). Second, when a new evaluation policy ($\pieval \neq \pi_b$) is used, the step-wise user density ratio $\rho_t$ forces an exponential variance explosion under covariate shift (\Cref{subsection: controllability collapse}).

\subsection{Look-Ahead Bias in Trajectory-Conditioned Training}
\label{subsection: lookahead bias}

To understand why this training paradigm breaks down, we first examine how a control label $\zsim$ is generated. In post-hoc labeling, the label is assigned after observing the entire conversation, creating a dependency where the label
implicitly encodes the agent's choices.

\begin{definition}[Action-dependent Labeling]
A labeling function $\PL(\cdot \mid \tau)$ is \emph{action-dependent} if for some step $t$ and realizations $h_{t-1}, u_t, a_t$, we have $P^{\pi_b}(\zsim \mid h_{t-1}, u_t, a_t) \neq P^{\pi_b}(\zsim \mid h_{t-1}, u_t)$.
\end{definition}

The above implies $\zsim \not\perp \!\!\!\perp a_t \mid \mathcal{F}_{t-}$. If the agent's actions influence this trajectory, conditioning the simulator's \emph{past} generations on this \emph{future-informed} label leaks information about the future. This violation of the natural filtration guarantees that the simulated distribution drifts from reality. 

\begin{theorem}[Divergence of Trajectory-Conditional Simulators]
\label{thm:lookahead}
Let $\PL$ be an action-dependent labeling function. For any policy $\pi$, the simulated trajectory density deviates from the true density by the following compounding factor:
\begin{align}\label{eq_thrm1}
\frac{\Psim^\pi(\tau)}{P^\pi(\tau)}
=
\mathbb{E}_{\zsim \sim \PL(\zsim\mid \tau)} \left[
\prod_{t=1}^T \frac{P^{\pi_b}(\zsim \mid h_{t-1}, u_{t})}{P^{\pi_b}(\zsim \mid h_{t-1}, u_{t}, a_{t})} \right].
\end{align}
\end{theorem}
Because $\zsim$ depends on the agent's future actions, the numerator and denominator generally do not cancel out.  Critically, the evaluation policy $\pieval$ is absent on the r.h.s. of \Cref{eq_thrm1}. The bias is entirely a function of the behavior policy $\pi_b$: the simulator marginalizes over $\pi_b$'s hypothetical actions (numerator), while the true environment updates using the realized action $a_t$ (denominator). Unless the label is completely independent of all agent actions, this ratio strictly diverges from 1.

This phenomenon structurally mirrors the \emph{conditioning bias} (or hindsight bias) well-documented in return-conditioned offline reinforcement learning, where conditioning on future outcomes breaks the natural filtration of the environment \citep{paster2022you,brandfonbrener2022does,wang2024critic,dou2025simulatorarena,naous2026flipping}.

\textbf{Example: A Recommender Agent.} To illustrate this bias, consider an environment where a user makes a vague request $u_1$. The agent recommends item $a_1=0$ or $a_1=1$, and the user responds by purchasing ($y=1$) or leaving ($y=0$). Let our trajectory control be a successful purchase: $\PL(\zsim=1 \mid \tau) = \mathbb{I}[y=1]$. 
In the real world, suppose item 1 succeeds 80\% of the time ($P(y=1 \mid u_1, a_1=1)=0.8$) and item 0 succeeds 20\% ($P(y=1 \mid u_1, a_1=0)=0.2$). If our offline dataset used an agent that tried both equally ($\pi_b(a_1=1 \mid u_1)=0.5$), the simulator learns a prior success rate of 50\%: $\Psim^{\pi_b}(\zsim=1) = 0.5(0.2) + 0.5(0.8) = 0.5$.

Now suppose the simulator is used to evaluate a new, optimized agent $\pieval$ that \emph{always} makes the high-reward recommendation ($\pieval(a_1=1 \mid u_1)=1$). Under $\pieval$ the true probability of a successful trajectory is $0.8$. However, because the simulator is anchored to the offline prior, the simulated success probability remains $0.5$. The simulator is functionally blind to the improvement in $\pieval$. Per \Cref{thm:lookahead}, this divergence is captured exactly by our bias factor: $\frac{P^{\pi_b}(\zsim=1 \mid u_1)}{P^{\pi_b}(\zsim=1 \mid u_1, a_1=1)} = \frac{0.5}{0.8}=0.625$.

\subsection{Controllability Collapse in Counterfactual Simulation}
\label{subsection: controllability collapse}

So far we've seen how trajectory-conditioned control can lead to bias. Now we examine how it can lead to disastrously compounding variance under policy shifts. We isolate this flaw by examining the cumulative validity ratio of simulated user actions: $W^{(u)}_T(\zsim) \triangleq \prod_{t=1}^T \rho_t(u_t \mid h_{t-1}, \zsim)$. By expanding the step-wise user generative error $\rho_t$ via Bayes' rule, the unconditional user dynamics $P(u_t \mid h_{t-1})$ cancel out perfectly. This reveals that the error is driven \emph{solely} by the mismatch in Bayesian belief updates regarding the label $\zsim$ under the two different policies. Let $M_t^\pi(u_t) \triangleq \frac{P^\pi(\zsim \mid h_{t-1}, u_t)}{P^\pi(\zsim \mid h_{t-1})}$ denote this multiplicative belief update. The step-wise mismatch is exactly the ratio of these updates:
$\rho_t(u_t \mid h_{t-1}, \zsim) = \frac{M_t^{\pieval}(u_t)}{M_t^{\pi_b}(u_t)}.$

Because this mismatch fluctuates at each conversation turn, we define its step-wise volatility as the \emph{local label sensitivity}: $V_t(h_{t-1}, \zsim, \pieval) \triangleq \text{Var}_{u_t \sim \Psim(\cdot \mid h_{t-1}, \zsim)} ( \rho_t(u_t \mid h_{t-1}, \zsim) )$.

\begin{theorem}[Geometric Explosion of Variance]
\label{thm:variance_explosion}
The counterfactual evaluation variance equals the cumulative sum of local label sensitivities, weighted by the compounding density ratio $W^{(u)}_{t-1}(\zsim)^2$:
$$\text{Var}_{\tau \sim \Psim^{\pieval}(\cdot \mid \zsim)} \left( W^{(u)}_T(\zsim) \right) = \sum_{t=1}^T \mathbb{E}_{h_{t-1} \sim \Psim^{\pieval}(\cdot \mid \zsim)} \left[ W^{(u)}_{t-1}(\zsim)^2 V_t(h_{t-1}, \zsim, \pieval) \right]$$
If local sensitivity is bounded below ($V_t(h_{t-1}, \zsim, \pieval) \ge \eta > 0$), variance explodes geometrically, strictly lower-bounded by $(1+\eta)^T - 1$.
\end{theorem}

\paragraph{The ``Perfect Simulator'' Paradox and the Curse of Horizon.}
\Cref{thm:variance_explosion} exposes a structural paradox directly analogous to the curse of horizon in classic OPE literature \citep{liu2018breaking,liu2020understanding,lai2026estimating}. Because evaluation tests a novel policy ($\pieval \neq \pi_b$), the agent's action distribution must shift. This guarantees belief update misalignment, forcing $V_t \ge \eta > 0$. Even a ``perfect'' text generator suffers geometric variance explosion simply because evaluation diverges from behavior. Testing a new policy directly breaks this metric (demonstrated in \Cref{sec:policy_misalignment_example} and proved in \Cref{sec:martingale_mechanics}).

\section{Restoring Causal Consistency}
\label{sec:solutions}

As standard trajectory-conditioned controls are prone to look-ahead bias and controllability collapse, we must reframe how simulation is controlled. We propose three \emph{training-time interventions} which fundamentally change the variables on which we condition our training objective. 

\textbf{1. A-Priori Independence:}
To preserve causal filtration, we restrict the global control $z$ to be strictly $\mathcal{F}_0$-measurable. Instead of post-hoc labels, we condition only on information fixed \emph{before} the interaction, such as demographics, initial explicit intents, or latent cognitive profiles. Relying solely on these pre-interaction variables reduces the fractional bias multiplier in \Cref{thm:lookahead} to exactly 1.

\textbf{2. Step-wise Dynamic Controls:}
If trajectory-level constraints are not strictly required, we can avoid policy dependence entirely by fundamentally shifting the control objective from achieving a predefined global outcome to modeling the user's actual, evolving state. Instead of extracting fixed labels from full trajectories, we generate a \emph{step-wise dynamic control} state $z_t$ (e.g., actual turn-level emotional affect, cognitive load, or changing goal) at each turn, based \emph{only} on the observable history up to that point: $z_t \sim \PL(\cdot \mid h_{t-1}, u_t)$. Because $z_t$ is computed before the agent selects $a_t$, it is properly adapted to $\mathcal{F}_{t-}$. This provides a formal guarantee: $z_t$ is conditionally independent of the agent's future actions, which reduces the fractional bias multiplier in \Cref{thm:lookahead} exactly to 1 and enables unbiased sequential sampling. During offline data preparation, an LLM annotator labels states $z_t$ for each sub-trajectory up to that turn. A single generative simulator then learns the joint distribution $P(z_t, u_t \mid h_{t-1}) = P(z_t \mid h_{t-1}) P(u_t \mid h_{t-1}, z_t)$ via autoregressive maximum likelihood over the entire sequence. By first generating $z_t$ during inference, the simulator safely avoids the look-ahead bias of post-hoc labels.

While purely autoregressive state generation preserves natural filtration, evaluators often wish to steer long-horizon behavioral patterns (e.g., testing de-escalation against an easily frustrated user) without reverting to biased post-hoc outcome labels. By composing Mitigations 1 and 2, we formulate a \emph{bipartite architectural framework} controlled by two strictly $\mathcal{F}_0$-measurable variables: an initial state $z_0$ (anchoring baseline mood or starting intent) and a global dynamics profile $c$ (governing how the user's state updates in response to agent actions). At each turn $t \ge 1$, a state dynamics kernel updates the user's latent state $P_{\text{dyn}}(z_t \mid h_{t-1}, z_{t-1}, c)$, after which a response kernel generates the utterance $P_{\text{resp}}(u_t \mid h_{t-1}, z_t)$. Because $(z_0, c) \in \mathcal{F}_0 \subseteq \mathcal{F}_{t-}$, the joint control remains strictly adapted to the filtration, guaranteeing zero look-ahead bias and zero local label sensitivity ($V_t = 0$) while enabling targeted stress-testing (formalized in \Cref{thm:dual_model_stability} with empirical de-escalation benchmarks in \Cref{sec:parameterized_dynamics}).

\textbf{3. Direct Policy-Conditioned Learning:}
To resolve controllability collapse under covariate shift (\Cref{thm:variance_explosion}) while maintaining global trajectory controls, we can explicitly condition the user simulator on the target agent policy itself, learning $P(u_t \mid h_{t-1}, z, \pi)$. By embedding the target policy's parameters or system prompt into the user simulator during training, the generative model explicitly internalizes the policy dependence. This aligns the modeled belief updates with the evaluation policy $\pieval$, empirically aligning the modeled belief updates between $M_t^{\pieval}$ and $M_t^{\pi_b}$ and relying on the network's out-of-distribution generalization to mitigate the variance explosion.

\section{Experimental Evaluation}
\label{sec:experiments}

To empirically validate our theoretical findings, we structure our evaluation into two distinct phases: (1) an \emph{in-distribution} evaluation isolating the distortion caused by look-ahead bias (\Cref{thm:lookahead}), and (2) an \emph{out-of-distribution} evaluation isolating variance explosion under covariate shift (\Cref{thm:variance_explosion}). 

\subsection{Experimental Setup}
\label{sec:experimental_setup}

\textbf{Datasets.} We evaluate on two multi-turn datasets. First, we use \emph{WildChat} \citep{zhao2024wildchat} (filtered to multi-turn conversations) to test general dynamics. We also introduce \emph{ConvApparel-V2},\footnote{Link: https://huggingface.co/datasets/google/ConvApparel.} an extension of ConvApparel \citep{meshi2026convapparel}, introducing newly collected human interactions with \emph{ten distinct assistant personas} for the footwear category, each governed by a system prompt reflecting specific styles (e.g., domain expert, efficient matchmaker, etc.). Details are in \Cref{sec: convapparel prompts,sec:data_collection}.

\textbf{Control Variables and Annotation.} 
To construct training controls, we use Gemini 3.1 Pro \citep{gemini3} as an LLM annotator to label the offline datasets. We extract three types of controls based on full trajectories: \textit{Persona+Goal} (extracting explicit linguistic rules and tasks), \textit{Cognitive Profile} (extracting latent traits like system literacy), and \textit{Scenario Generation} (a holistic summary of how the interaction progressed). To isolate causal bias from simple 
LLM-induced
verbosity (e.g., generating more turns simply because of a long prompt), we also create a \textit{Length-Constrained Scenario} control that prompts the simulator to respect a strict length limit. For our dynamic mitigation, we annotate turns with \emph{step-wise dynamic} controls representing turn-level emotional affect, cognitive load, and implicit intent for each sub-trajectory. Complete prompts used for all annotations are described in \Cref{sec:prompt_templates}.

\textbf{Baselines.} 
We consider the following simulation baselines:
(1) \emph{Unconditioned SFT:} A baseline simulator SFT-trained on user tokens with no controls; (2) \emph{Prompted Simulator:} A zero-shot baseline where an off-the-shelf LLM is prompted with the global trajectory controls to act as the user; (3) \emph{Rejection Sampling:} An inference-time baseline that generates $N$ candidate trajectories from the Unconditioned SFT model and selects the one that best matches the target control (as determined by an LLM judge); (4) \emph{Trajectory-conditioned SFT (standard practice):} A simulator SFT-trained to maximize the likelihood of user tokens conditioned on post-hoc global controls (Persona+Goal, Scenario) prepended as a system prompt; (5) \emph{Cognitive Profile SFT (mitigating \Cref{thm:lookahead}):} Trajectory-conditioned SFT restricted to a cognitive profile constraint, attempting to avoid action-dependent look-ahead bias; (6) \emph{Dynamic State SFT (mitigating \Cref{thm:lookahead}):} Trained to predict the intermediate dynamic state annotation $z_t$ immediately before generating the user text $u_t$. (7) \emph{Agent-Aware SFT (mitigating \Cref{thm:variance_explosion}):} To specifically mitigate out-of-distribution controllability collapse, an SFT-trained model conditioned on both the user control and the target agent's system prompt.

\textbf{Automated Evaluation Loop.} To automate closed-loop offline testing, we use proxy SFT agents. For WildChat, we train an unconditioned agent on the dataset's empirical assistant responses. For ConvApparel, we train an agent conditioned on distinct assistant prompts, parameterizing the policy shift $\pieval$ by swapping prompts at inference time.

\textbf{Evaluation Metrics.}
We evaluate simulation quality across three dimensions to understand the tradeoffs involved in controllable simulation:
(1) \emph{Distributional Fidelity \& Semantic Drift}: Does the simulator behave like a natural human population, or do the controls induce unnatural conversations? We assess both deterministic statistics (turn counts, word lengths, etc.) and semantic intent;
(2) \emph{Control Adherence}: Does the simulator successfully reflect the requested controls? We measure this via an independent LLM-as-a-judge (scoring Persona and Goal match) as well as Gemini text-embedding cosine-similarity;
(3) \emph{Behavioral Diversity}: Do the generated responses exhibit natural human variance, or do they collapse into repetitive patterns? We quantify this by measuring Shannon entropy of the discrete adherence scores (evaluated by the LLM judge on a categorical scale).

\begin{table}[t!]
\centering
\caption{\footnotesize \textbf{In-Distribution Generative Fidelity and Semantic Drift (WildChat).} Unconditioned models truncate conversations; globally controlled models artificially inflate turns and suffer semantic drift. Bold indicates the conditioned simulator closest to the human baseline. Full metrics are provided in Appendix \ref{sec:extended_in_distribution}.}
\label{tab:in_dist_merged}
\renewcommand{\arraystretch}{1.2}
\resizebox{\textwidth}{!}{%
\begin{tabular}{lccccc}
\toprule
\textbf{Simulator Type} & \textbf{Turn Count} & \textbf{Avg. Words / Turn} & \textbf{Task Unresolved} & \textbf{Iterative Refine.} & \textbf{Highly Detailed} \\
\midrule
Ground Truth Data & 4.31 $\pm$ 0.01 & 113.3 $\pm$ 0.1 & 85.4\% & 48.4\% & 11.5\% \\
\midrule
Unconditioned SFT & 3.55 $\pm$ 0.01 & 85.6 $\pm$ 0.2 & 91.1\% & 32.3\% & 11.3\% \\
\midrule
Prompted Simulator (Cognitive) & 6.45 $\pm$ 0.01 & 88.3 $\pm$ 0.2 & 89.1\% & 19.2\% & 28.4\% \\
Prompted Simulator (Scenario) & 8.62 $\pm$ 0.02 & 68.5 $\pm$ 0.2 & 87.4\% & 17.8\% & 41.2\% \\
Prompted (Scenario, Length-Constr.) & 4.45 $\pm$ 0.01 & 92.2 $\pm$ 0.2 & 88.0\% & 19.5\% & 38.2\% \\
\midrule
Cognitive Profile Conditioned SFT & 5.19 $\pm$ 0.01 & 95.1 $\pm$ 0.2 & \textbf{83.8\%} & 29.1\% & \textbf{14.1\%} \\
Trajectory-Cond. SFT (Scenario) & 7.37 $\pm$ 0.01 & 81.4 $\pm$ 0.1 & 83.2\% & 28.7\% & 24.5\% \\
Dynamic State Conditioned SFT (Ours) & \textbf{4.78 $\pm$ 0.01} & \textbf{109.5 $\pm$ 0.2} & 76.8\% & \textbf{30.2\%} & 21.1\% \\
\bottomrule
\end{tabular}%
}
\end{table}

\subsection{In-Distribution: Testing Look-Ahead Bias and Adherence}
\label{sec:in_distribution_eval}

We first assess whether models adhere to controls without distorting the natural user behavior distribution, evaluating \Cref{thm:lookahead} on the WildChat dataset (in-distribution results for ConvApparel exhibit the same trends and are described in \Cref{sec:detailed_eval_tables}). We present a condensed summary of key linguistic and behavioral metrics in \Cref{tab:in_dist_merged}; complete tables with all metrics, adherence scores, and entropies are provided in Appendix \ref{sec:extended_in_distribution}.

\textbf{Length Inflation and Semantic Drift.} Standard trajectory-conditioned SFT models structurally over-correct due to look-ahead bias. Because the label drives toward an outcome from the end of the offline trajectory, the simulator artificially prolongs interactions to ensure constraints are met. For example, Scenario-Conditioned SFT averages 7.37 turns, and the Prompted version hits 8.62, compared to the human baseline of 4.31. 
This global bias triggers severe semantic drift. Real users construct ``Highly Detailed'' prompts 11.5\% of the time, whereas standard trajectory-conditioned models double this rate ($\sim$24.5\%) by forcibly injecting constraints upfront. Furthermore, they drastically under-use Iterative Refinement (28.7\% vs.\ human 48.4\%, \Cref{tab:in_dist_merged}). This suggests that the simulator abandons natural user behavior in favor of rigid, transactional compliance simply to satisfy the global control.

We verify this to be a structural causal violation and not mere LLM verbosity. When our \emph{Length-Constrained Scenario} baseline explicitly forces the global prompt to respect a 4-turn limit, it successfully caps the turn count at 4.45. However, semantic drift remains severe (38.2\% Highly Detailed prompts and 19.5\% Iterative Refinement). Conversely, our \emph{Dynamic State SFT} simulator resolves these distortions. By generating states purely autoregressively, it prioritizes step-wise coherence and does not artificially enforce long-horizon global outcomes. This represents a fundamental trade-off: bypassing post-hoc labels eliminates semantic drift and length inflation, but trades away strict global controllability (unless utilizing explicitly parameterized state dynamics, as discussed in \Cref{sec:parameterized_dynamics}).

\begin{figure}[t!]
\begin{minipage}{0.48\textwidth}
\centering
\includegraphics[width=\linewidth]{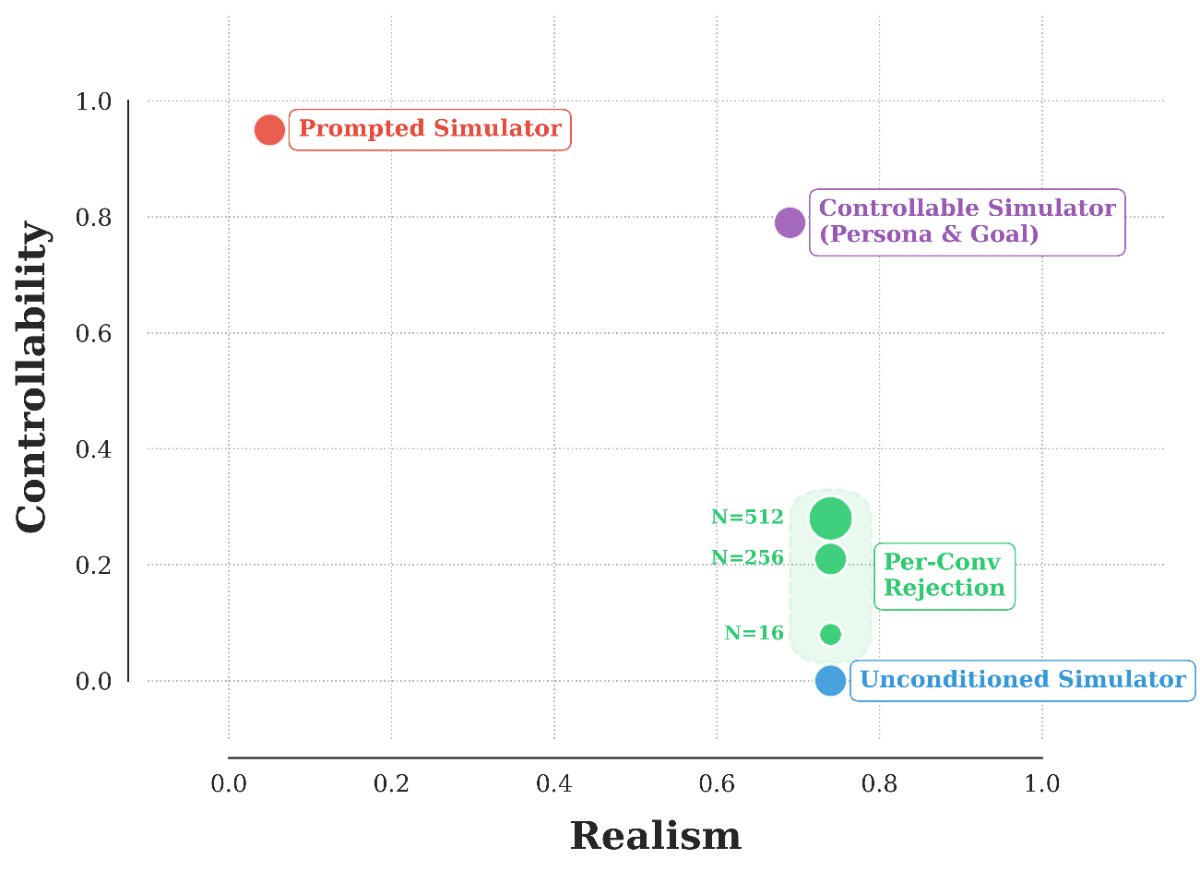}
\caption{\footnotesize \textbf{Realism-Controllability Trade-off.} Explicitly prompted models achieve high adherence but suffer severe diversity collapse. Best-of-N sampling fails completely. Causally conditioned models offer a superior Pareto frontier.}
\label{fig: realism controllability tradeoff}
\end{minipage}
\hfill
\begin{minipage}{0.48\textwidth}
\captionof{table}{\footnotesize \textbf{Control Adherence Scores.} Mean $\pm$ 95\% CI via LLM Judge and embedding cosine-similarity.}
\label{tab:adherence_scores}
\renewcommand{\arraystretch}{1.2}
\resizebox{\linewidth}{!}{%
\begin{tabular}{lcc}
\toprule
\textbf{Simulator Type} & \textbf{Persona Match} & \textbf{Embed. Sim.} \\
\midrule
Ground Truth (Persona+Goal) & $0.89 \pm 0.01$ & $0.87 \pm 0.01$ \\
Prompted Sim (Persona+Goal) & $0.96 \pm 0.01$ & $0.93 \pm 0.01$  \\
Traj-Cond. SFT (Persona+Goal) & $0.77 \pm 0.01$  & $0.69 \pm 0.01$ \\
\midrule
Ground Truth (Scenario) & $0.92 \pm 0.01$ & $0.88 \pm 0.01$ \\
Prompted Sim (Scenario) & $0.85 \pm 0.01$ & $0.82 \pm 0.01$ \\
Traj-Cond. SFT (Scenario) & $0.82 \pm 0.01$ & $0.69 \pm 0.01$ \\
\midrule
Ground Truth (Cognitive) & $0.72 \pm 0.01$ & $0.70 \pm 0.01$ \\
Prompted Sim (Cognitive) & $0.90 \pm 0.01$ & $0.88 \pm 0.01$ \\
Cognitive Profile SFT & $0.56 \pm 0.01$ & $0.52 \pm 0.01$ \\
\bottomrule
\end{tabular}%
}
\end{minipage}
\end{figure}

\textbf{Constraint Adherence and Over-Compliance.}
Agent evaluators need simulators to faithfully follow targeted constraints. As seen in \Cref{tab:adherence_scores}, explicitly prompted simulators achieve high raw adherence scores ($0.90$ Persona Match on cognitive profiles vs.\ $0.72$ for human ground truth). Rather than indicating superior realism, scoring above the human baseline reflects \emph{artificial over-compliance}: human behavior is naturally varied, whereas prompted LLMs force generation into rigid, stereotypical paths to satisfy instructions. We verify that this pattern is not an artifact of self-preference in the Gemini judge by replicating the evaluation with an independent GPT-5.6 judge (\Cref{sec:cross_family_appendix}), which confirms the exact same over-adherence phenomenon ($0.86$ prompted vs.\ $0.58$ human on cognitive profiles). Our trajectory-conditioned SFT models show reasonable adherence for explicit tasks ($0.82$ Persona Match for Scenarios), while text-based SFT on abstract cognitive profiles displays lower adherence ($0.56$). Finally, because purely autoregressive \emph{Dynamic State SFT} operates without a static trajectory target, it is excluded from \Cref{tab:adherence_scores}; we evaluate targeted dynamic reactivity via our bipartite framework in \Cref{sec:parameterized_dynamics}, and provide side-by-side qualitative dialogue transcripts in \Cref{sec:qualitative_appendix}.

\textbf{Behavioral Diversity Collapse.}
While global controls can achieve high adherence, they fundamentally distort natural behavioral diversity. We quantify this by measuring the Shannon entropy over the discrete adherence scores (see full results in \Cref{sec:extended_in_distribution,tab:shannon_entropy_full}). Explicitly prompted (non-SFT-trained) simulators severely compress variance of real human behavior to adhering to complex controls like Persona+Goal (driving entropy down from the human baseline of 1.04 to 0.728) and cognitive profiles (from 1.10 down to 0.636). The models collapse to finding a ``safe" generative path to satisfy the prompt, overriding the natural diversity of human behavior. Conversely, trajectory-conditioned SFT models artificially inflate this variance (entropies ranging from 1.28 to 1.64), but do so at the cost of structural length inflation and semantic drift as discussed above.

Attempting to avoid this adherence-diversity trade-off via inference-time \emph{Rejection Sampling} fails completely (see \Cref{fig: realism controllability tradeoff}). Searching the vast combinatorial space of multi-turn dialogue is computationally intractable, yielding near-zero adherence even with massive compute overhead ($N=512$ candidates per conversation, see further results in Appendix \ref{sec:rejection_sampling_appendix}).

\begin{table}[t!]
\centering
\caption{\footnotesize \textbf{Condensed Counterfactual Stability under Covariate Shift.} Evaluated zero-shot on unseen agent policies. The standard Static Agnostic approach collapses when facing the terse ``Efficient Matchmaker''. (Full comprehensive metrics for all four held-out agents are provided in \Cref{sec:detailed_eval_tables}).}
\label{tab:off_policy_condensed}
\renewcommand{\arraystretch}{1.1}
\resizebox{\textwidth}{!}{%
\begin{tabular}{llcccc}
\toprule
\multirow{2}{*}{\makecell[l]{\textbf{Held-Out}\\\textbf{Evaluation Agent}}} & \multirow{2}{*}{\makecell[l]{\textbf{Simulator}\\\textbf{Paradigm}}} & \multicolumn{2}{c}{\textbf{Linguistic Statistics}} & \multicolumn{2}{c}{\textbf{Behavioral Intent (\%)}} \\
\cmidrule(lr){3-4} \cmidrule(lr){5-6}
 & & \makecell[c]{\textbf{Avg. Words}\\\textbf{per Turn}} & \makecell[c]{\textbf{Total User}\\\textbf{Words}} & \makecell[c]{\textbf{Negative}\\\textbf{Sentiment}} & \makecell[c]{\textbf{Iterative}\\\textbf{Refinement}} \\
\midrule
\multirow{5}{*}{\makecell[l]{Domain Expert\\(Academic, Verbose)}} 
& Ground Truth Data & 12.0 & 38.7 & 13.4 & 31.7 \\
& Static, Agnostic (Standard) & 8.8 & 28.0 & 18.0 & 12.5 \\
& Static, Aware & 10.5 & 32.0 & 15.1 & 25.0 \\
& Dynamic, Agnostic & 11.2 & 35.1 & 14.5 & 22.3 \\
& Dynamic, Aware (Ours) & \textbf{12.1} & \textbf{38.5} & \textbf{13.6} & \textbf{31.2} \\
\midrule
\multirow{5}{*}{\makecell[l]{Efficient Matchmaker\\(Ultra-terse, Fast)}} 
& Ground Truth Data & 11.7 & 39.7 & 11.9 & 41.3 \\
& Static, Agnostic (Standard) & 25.4 & 92.1 & 26.8 & 8.8 \\
& Static, Aware & 13.2 & 43.5 & 16.4 & 32.6 \\
& Dynamic, Agnostic & 14.1 & 45.2 & 18.2 & 24.8 \\
& Dynamic, Aware (Ours) & \textbf{11.8} & \textbf{40.1} & \textbf{12.2} & \textbf{40.9} \\
\bottomrule
\end{tabular}%
}
\end{table}

\begin{table}[t!]
\centering
\caption{\footnotesize \textbf{Empirical Validation of Variance Explosion (\Cref{thm:variance_explosion}).} Normalized variance ratio $\text{Var}_{\text{sim}} / \text{Var}_{\text{GT}}$ of Total User Words, stratified by conversation horizon $T$ across all four held-out ConvApparel agents. Ground truth variance is computed from real human data. A ratio of $1.0$ indicates perfect alignment with real human variance.}
\label{tab:variance_explosion}
\renewcommand{\arraystretch}{1.2}
\resizebox{\textwidth}{!}{%
\begin{tabular}{lcccccl}
\toprule
\textbf{Simulator Paradigm} & $T=2$ & $T=3$ & $T=4$ & $T=5$ & $T \ge 6^\ast$ & \textbf{Avg. Step-wise Growth ($\hat{\eta}$)} \\
\midrule
Ground Truth Data & $1.00$ & $1.00$ & $1.00$ & $1.00$ & $1.00$ & --- \\
\midrule
Static Global, Agent-Agnostic & $1.76$ & $2.36$ & $3.52$ & $4.89$ & $7.19$ & $\hat{\eta} \approx 0.42$ \\
Static Global, Agent-Aware & $1.28$ & $1.34$ & $1.35$ & $1.81$ & $2.14$ & $\hat{\eta} \approx 0.14$ \\
Dynamic State, Agent-Agnostic & $1.14$ & $1.15$ & $1.29$ & $1.35$ & $1.44$ & $\hat{\eta} \approx 0.06$ \\
Dynamic State, Agent-Aware (Ours) & $\mathbf{1.03}$ & $\mathbf{0.99}$ & $\mathbf{1.05}$ & $\mathbf{1.02}$ & $\mathbf{1.07}$ & $\hat{\eta} \approx 0.01$ \\
\midrule
\multicolumn{7}{l}{\footnotesize \textit{Reference:} $\text{Var}_{\text{GT}}(\text{Total Words})$: $T\!=\!2$: $185$; \; $T\!=\!3$: $419$; \; $T\!=\!4$: $606$; \; $T\!=\!5$: $975$; \; $T\!\ge\!6^\ast$: $4{,}835$} \\
\multicolumn{7}{l}{\footnotesize $^\ast$ \textit{Note:} The $T \ge 6$ column aggregates the long-tail of conversations. We treat this conceptually as $T=6$ for growth estimates.} \\
\bottomrule
\end{tabular}%
}
\end{table}

\subsection{Out-of-Distribution: Testing Controllability Collapse}
\label{sec:ood_eval}

To test whether simulators suffer from controllability collapse under covariate shift ($\pi_e \neq \pi_b$), we construct a train/test split utilizing the 13 footwear agents. We train our controllable user simulator via SFT on interactions with the nine in-distribution agents, and evaluate them zero-shot in closed-loop conversations with the four strictly held-out personas.

\Cref{thm:variance_explosion} bounds the variance of the trajectory density ratio (the importance sampling weights). As a proxy to assess this breakdown empirically, we measure the variance of an additive trajectory feature: \emph{total user words}. In an uncorrected simulator, the compounding instability of step-wise density ratios manifests as extreme trajectory-level fluctuations in such additive metrics. \Cref{tab:off_policy_condensed} and \Cref{tab:variance_explosion} confirm this breakdown: these results expose both severe mean distortion and compounding variance instability.\footnote{While word count is an additive feature bounded by maximum length, its rapid step-wise variance growth serves as a strong empirical proxy for the density ratio's theoretical collapse. For the interested reader, we provide additional direct empirical validation of the true multiplicative density ratio collapse using offline classifiers in \Cref{sec:direct_density_validation}.} When standard trajectory-conditioned SFT (Static, Agent-Agnostic) encounters the terse ``Efficient Matchmaker,'' it inflates word counts ($25.4$ vs.\ $11.7$ human) and negative sentiment (26.8\% vs.\ 11.9\% human). Conversely, with the verbose ``Domain Expert," it unnaturally compresses responses. Since the policy-agnostic simulator produces heterogeneous mean errors across different policies, this failure compounds multiplicatively over time. As a result, its aggregate normalized variance ($\text{Var}_{\text{sim}} / \text{Var}_{\text{GT}}$) explodes, growing by an average factor of $\sim 1.42\times$ \emph{at each turn} ($\hat{\eta} \approx 0.42$).

Explicitly resolving these structural discrepancies restores stability. Employing the \emph{Dynamic State} formulation avoids look-ahead bias by adapting only to the step-wise observable history, while \emph{Agent-Aware SFT} directly resolves any policy-induced mismatched expectations. The combined \emph{Dynamic State, Agent-Aware} approach achieves near-perfect human parity ($11.8$ words/turn), mirrors human subtle behavioral adaptations to new styles (e.g., matching negative sentiment at 12.2\% with the Matchmaker) and maintains a highly stable variance ratio across conversation horizons ($\hat{\eta} \approx 0.01$). 

While total user words serve as a clear downstream proxy for this failure, \Cref{thm:variance_explosion} strictly predicts the geometric explosion of the underlying mathematical trajectory density ratio. We provide a direct empirical validation of this probabilistic collapse (evaluating $P(z \mid h_t)$ via offline classifiers) in \Cref{sec:direct_density_validation} (\Cref{tab:variance_explosion_density}), alongside granular turn-by-turn negative emotion incidence and confidence intervals for long-tail horizons ($T \ge 6$) in \Cref{sec:detailed_eval_tables}. Furthermore, to ensure our findings are not specific to the Gemini model family or clean synthetic annotations, we replicate both the WildChat and ConvApparel density-ratio experiments using open-weight \textbf{Qwen3} simulators and annotators (\Cref{sec:cross_family_appendix}), and demonstrate via synthetic label-corruption ablations that dynamic state training remains highly stable even under 30\% annotator label noise (\Cref{sec:annotator_noise_ablation}).

\textbf{Policy Visibility Requirements.}
Notably, direct policy conditioning (Agent-Aware SFT) requires only the target agent's system prompt or high-level style directives, not model weights or internal activations. For black-box commercial APIs or multi-stage RAG pipelines where prompts are hidden, evaluators can summarize observed agent behavior offline into a lightweight ``policy card'' prompt. Crucially, when zero visibility into the agent policy is desired, step-wise dynamic control natively bypasses policy conditioning altogether: by conditioning strictly on the observable filtration $\mathcal{F}_{t-}$, it holds variance bounded ($1.44$ at $T \ge 6$ vs.\ $7.19$ for standard SFT) with zero knowledge of the agent's architecture.

\section{Conclusion}
\label{sec:conclusion}

Reliable offline evaluation requires simulating counterfactuals without deviating from natural user distributions. We show that in standard SFT simulators, static trajectory labels inject \emph{look-ahead bias} violating interaction filtration. This causes distributional distortions, diversity collapse, and generative breakdown under policy shifts. Our proposed causally grounded mitigations ($\mathcal{F}_0$-measurable traits, step-wise dynamic controls, and direct policy-conditioned learning) empirically eliminate this bias, preserve natural variance, and enable robust zero-shot generalization. 

While these mitigations provide a rigorous foundation, several avenues for future work remain. Adapting direct policy conditioning (which currently assumes white-box agent access) for black-box environments like proprietary APIs is a compelling next step. Furthermore, extending our dynamic state tracking to enforce explicit long-horizon constraints opens a rich design space. Ultimately, integrating representation and reinforcement learning could automate the discovery of optimal dynamic control profiles, allowing evaluators to seamlessly steer complex user behaviors while preserving causal fidelity.

\newpage
\bibliography{bibliography}
\bibliographystyle{plainnat}

\newpage
\appendix
\crefalias{section}{appendix}

\section{Discussion: Enforcing Global Constraints via Parameterized State Dynamics}
\label[appendix]{sec:parameterized_dynamics}

In the main text, we established that step-wise dynamic control elegantly resolves the mathematical pitfalls of off-policy evaluation by generating states conditionally on the observable history. However, predicting the next local state purely autoregressively leaves open the question of how to exert targeted, global influence over the trajectory. If an evaluator wishes to simulate a user whose emotional state degrades predictably when an agent is overly repetitive, standard step-wise generation lacks an explicit mechanism to enforce this trajectory-level constraint without reverting to flawed post-hoc labels. To enforce global behavioral shapes while strictly avoiding future-leaking outcome conditioning, we propose a bipartite architectural framework that explicitly parameterizes the transition dynamics.

This framework is realized by decoupling the user simulator into two distinct generative components: a state dynamics model and a user response model. Instead of conditioning generation on an eventual trajectory outcome, we define two control variables that are strictly adapted to the initial filtration. The first is an initial state variable $z_0$, which anchors the user's starting intent or emotional baseline. The second is a global dynamics control parameter $c$. This parameter acts as an overarching profile or an embedded set of transition rules that governs the mechanics of how the user's internal state evolves in response to the agent. Because both $z_0$ and $c$ are determined strictly prior to the interaction, they possess no backward causal links to the agent's future actions.

The sequential generative process operates chronologically. At any time step $t \ge 1$, the user possesses a latent internal state $z_t$. This state is updated from the previous state by processing the observable history up to the previous turn and the structural laws dictated by the profile, formalized as the state dynamics model $P_{\text{dyn}}(z_t \mid h_{t-1}, z_{t-1}, c)$. Following this internal update, the user response model samples the conversational utterance conditionally based on the observable history and this current state, formalized as $P_{\text{resp}}(u_t \mid h_{t-1}, z_t)$. Finally, the target agent observes the history and selects an action according to its policy $\pi(a_t \mid h_{t-1}, u_t)$, producing the updated history $h_t = (h_{t-1}, u_t, a_t)$ and completing the cycle.

Through this decoupled architecture, evaluators control long-horizon behavior implicitly. Rather than prompting a simulator to end the conversation in frustration, an evaluator supplies a dynamics profile $c$ dictating that repeated agent misunderstandings deterministically degrade the user's patience state $z_t$, which the response model subsequently translates into hostile text. 

Below, we show that this architecture guarantees complete causal consistency and statistical stability under arbitrary off-policy covariate shift.

\begin{theorem}[Causal Consistency and Stability of Parameterized State Dynamics]
\label{thm:dual_model_stability}
Let a user simulation process be parameterized by strictly $\mathcal{F}_0$-measurable variables $z_0$ and $c$, alongside a state transition kernel $P_{\text{dyn}}(z_t \mid h_{t-1}, z_{t-1}, c)$ and a response kernel $P_{\text{resp}}(u_t \mid h_{t-1}, z_t)$. Assuming an unbiased simulator ($\text{P}_{\text{sim}} = \text{P}^{\pi_b}$), for any arbitrary evaluation policy $\pieval$ and behavior policy $\pi_b$, the look-ahead bias evaluates to exactly one, and the step-wise counterfactual user density ratio $\rho_t(u_t \mid h_{t-1}, z_0, c)$ evaluates to exactly one almost surely. Consequently, the local label sensitivity $V_t$ equates to zero, completely preventing geometric variance explosion.
\end{theorem}

\begin{proof}
We first address the look-ahead bias. By \Cref{thm:lookahead}, look-ahead bias is quantified by the ratio of the conditional probability of the control variables given the pre-action history to their probability given the post-action history. Because the controls $(z_0, c)$ are defined entirely prior to the conversation, they act as fixed causal ancestors. The agent selects its action $a_t$ according to a policy $\pi_b(a_t \mid h_{t-1}, u_t)$ that is solely a function of the observable history. Consequently, the unobserved global controls are conditionally independent of the action $a_t$ given the pre-action history, yielding $(z_0, c) \perp\!\!\!\perp a_t \mid h_{t-1}, u_t$. Applying Bayes' theorem to the post-action probability expands the denominator into a fraction where the likelihood term $\pi_b(a_t \mid h_{t-1}, u_t, z_0, c)$ simplifies directly to $\pi_b(a_t \mid h_{t-1}, u_t)$. Because this term is independent of the unobserved controls, it factors entirely out of the marginalization integral over the controls and perfectly cancels with the identical policy term in the numerator. The posterior probability thus mathematically collapses to the prior probability $P^{\pi_b}(z_0, c \mid h_{t-1}, u_t)$, forcing the bias ratio to equal one and structurally preventing any causal leakage from future events.

We next address controllability collapse by analyzing the user generative error $\rho_t(u_t \mid h_{t-1}, z_0, c)$, which fundamentally drives the variance explosion in \Cref{thm:variance_explosion}. We evaluate the conditional marginal probability of a user utterance $u_t$ given the observable history and the global controls under an arbitrary agent policy $\pi$, denoted as $P^\pi(u_t \mid h_{t-1}, z_0, c)$. By the definition of conditional probability, this is the ratio of the joint density of the utterance and the history over the marginal density of the history. To obtain these quantities, we integrate out the unobserved latent state trajectory $z_{1:t}$. Applying the chain rule of probability to the causal sequence of the environment, the marginal probability of the history up to step $t-1$ expands as the integral over the latent path $z_{1:t-1}$:
\begin{align}
P^\pi(h_{t-1} \mid z_0, c) &= \int \prod_{k=1}^{t-1} \Big[ P_{\text{resp}}(u_k \mid h_{k-1}, z_k) P_{\text{dyn}}(z_k \mid h_{k-1}, z_{k-1}, c) \pi(a_k \mid h_{k-1}, u_k) \Big] \, dz_{1:t-1}.
\end{align}
Because we condition our generative step on a strictly observed history sequence $h_{t-1}$, the specific sequence of past agent actions and user utterances are fixed constants in this context. Thus, the product of the agent's action probabilities contains no variables dependent on the integration variables $z_{1:t-1}$. This allows the entire policy product term to factor completely outside of the integral:
\begin{align}
P^\pi(h_{t-1} \mid z_0, c) &= \left( \prod_{k=1}^{t-1} \pi(a_k \mid h_{k-1}, u_k) \right) \int \prod_{k=1}^{t-1} P_{\text{resp}}(u_k \mid h_{k-1}, z_k) P_{\text{dyn}}(z_k \mid h_{k-1}, z_{k-1}, c) \, dz_{1:t-1}.
\end{align}
We apply this exact factorization to the numerator to calculate the joint density $P^\pi(u_t, h_{t-1} \mid z_0, c)$. The term representing the agent's past choices factors out of the integration over $z_{1:t}$ in an identical manner:
\begin{align}
P^\pi(u_t, h_{t-1} \mid z_0, c) &= \left( \prod_{k=1}^{t-1} \pi(a_k \mid h_{k-1}, u_k) \right) \int P_{\text{resp}}(u_t \mid h_{t-1}, z_t) P_{\text{dyn}}(z_t \mid h_{t-1}, z_{t-1}, c) \nonumber \\
&\quad \times \prod_{k=1}^{t-1} P_{\text{resp}}(u_k \mid h_{k-1}, z_k) P_{\text{dyn}}(z_k \mid h_{k-1}, z_{k-1}, c) \, dz_{1:t}.
\end{align}
Dividing the numerator by the denominator isolates $P^\pi(u_t \mid h_{t-1}, z_0, c)$. The factored policy product term appears identically in both the top and bottom of the fraction, completely canceling out. This analytical cancellation demonstrates that the true, marginalized probability of the user's action $u_t$ is purely a function of the structural transition kernels and is strictly independent of the agent policy $\pi$. Therefore, evaluating the true user dynamics under the evaluation policy $\pieval$ yields the exact same distribution as under the behavior policy $\pi_b$. Assuming an unbiased simulator that successfully captures this causal factorization, the step-wise density ratio $\rho_t(u_t \mid h_{t-1}, z_0, c)$ becomes perfectly deterministic and identically equal to one for all possible utterances $u_t$. Consequently, its variance across the step-wise generation, defined as the local label sensitivity $V_t$, evaluates to exactly zero. Substituting $V_t = 0$ into the recursive expectation bound established in \Cref{thm:variance_explosion} ensures that the total counterfactual variance structurally evaluates to zero, completing the proof.
\end{proof}

This bipartite dynamic framework fundamentally transforms controllable simulation from an outcome-generation paradigm into an inverse design problem. Evaluators cannot simply pass a target outcome to the model during inference and expect a causally sound interaction. Instead, creating targeted counterfactual datasets requires discovering the optimal combinations of initial states $z_0$ and dynamics profiles $c$ that naturally induce the desired outcomes when simulated against specific target policies. This opens a critical and highly complex avenue for future work in prompt optimization and representation learning. We envision future research exploring reinforcement learning applied to the state dynamics model's meta-parameters, continuous embedding optimization, or gradient-based discrete search to automatically discover the optimal dynamics control profiles. By shifting the objective from directly generating text that matches an outcome to optimizing the rules of an environment such that the desired outcome emerges naturally, the field can generate highly specific, stress-tested conversational datasets that strictly adhere to causal mathematics.

\subsection{Empirical Demonstration: De-Escalation Benchmarking}
\label[appendix]{sec:deescalation_benchmark}

To empirically demonstrate how the bipartite framework resolves the tension between targeted stress-testing and causal validity, we evaluate an agent de-escalation scenario across the four held-out ConvApparel agents. Rather than conditioning on a future-leaking outcome label ($\zsim = \text{``ends frustrated''}$), the evaluator specifies a pre-interaction starting state $z_0 \in \mathcal{F}_0$ (baseline mood) and a global dynamics reactivity profile $c \in \mathcal{F}_0$ (either an \emph{Average Reactivity} profile or a \emph{Low-Patience} profile where repeated agent misunderstandings or verbosity rapidly degrade the user's patience state $z_t$). We conduct 500 closed-loop rollouts per agent and measure the percentage of conversations in which user frustration emerges.

\begin{table}[h!]
\centering
\caption{\footnotesize \textbf{De-Escalation Benchmark (Frustration Emergence Rate \% $\pm$ 95\% CI).} Outcome-conditioned SFT forces high frustration regardless of agent quality. Under our bipartite framework ($c = \text{Low-Patience}$), frustration emerges dynamically as a true function of the target agent's de-escalation capability.}
\label{tab:deescalation_bipartite}
\renewcommand{\arraystretch}{1.2}
\resizebox{\textwidth}{!}{%
\begin{tabular}{lcccc}
\toprule
\textbf{Held-Out Agent} & \textbf{Human Ground Truth} & \textbf{Traj-Cond. ($\zsim =$ ``ends frustrated'')} & \textbf{Bipartite ($c =$ Avg. Reactivity)} & \textbf{Bipartite ($c =$ Low-Patience)} \\
\midrule
Domain Expert (Verbose, Dry) & $13.4 \pm 3.7\%$ & $77.4 \pm 3.7\%$ & $14.8 \pm 3.1\%$ & $38.6 \pm 4.3\%$ \\
Empathetic Listener (Validating) & $8.1 \pm 3.0\%$ & $78.0 \pm 3.6\%$ & $9.8 \pm 2.6\%$ & $9.8 \pm 2.6\%$ \\
Efficient Matchmaker (Ultra-terse) & $11.9 \pm 3.6\%$ & $75.6 \pm 3.8\%$ & $16.6 \pm 3.3\%$ & $32.4 \pm 4.1\%$ \\
Enthusiastic Rambler (Distracted) & $8.9 \pm 3.2\%$ & $75.2 \pm 3.8\%$ & $11.0 \pm 2.7\%$ & $14.6 \pm 3.1\%$ \\
\bottomrule
\end{tabular}%
}
\end{table}

As shown in \Cref{tab:deescalation_bipartite}, direct outcome conditioning ($\zsim = \text{``ends frustrated''}$) suffers from severe level saturation: it forces user frustration at roughly $6\times$ the human rate ($\sim 75\text{--}78\%$) uniformly across all agents, completely blinding the evaluator to differences in agent quality. Conversely, under the bipartite framework with a \emph{Low-Patience} profile, an easily frustrated user reaches $38.6\%$ frustration against the dry, verbose \emph{Domain Expert} and $32.4\%$ against the abrupt \emph{Efficient Matchmaker}, but stays at $9.8\%$ when interacting with the validating \emph{Empathetic Listener}. While causally valid simulation does not rigidly guarantee frustration upfront, conditioning on $(z_0, c)$ enables targeted stress-testing where emotional outcomes emerge authentically in response to agent behavior.

\section{Proof of Divergence in Trajectory-Labeled User Simulators}
\label[appendix]{sec:proof}

We analyze the explicit density ratio between the true trajectory distribution $P^\pi(\tau)$ under an arbitrary policy $\pi$ and the distribution $\Psim^\pi(\tau)$ generated by a causally trained controllable simulator. The learned simulator dynamics converge to the empirical conditional distribution: $\Psim(u_t \mid h_{t-1}, \zsim) \triangleq P^{\pi_b}(u_t \mid h_{t-1}, \zsim)$.

Using Bayes' theorem, we expand this conditional density. Because the user's base unconditional action $u_t$ depends only on the natural filtration $\mathcal{F}_{t-1}$ and not on the future behavior policy, $P^{\pi_b}(u_t \mid h_{t-1})$ reduces to $P(u_t \mid h_{t-1})$. Thus:
$$\Psim(u_t \mid h_{t-1}, \zsim) = P(u_t \mid h_{t-1}) \frac{P^{\pi_b}(\zsim \mid h_{t-1}, u_t)}{P^{\pi_b}(\zsim \mid h_{t-1})}$$

During evaluation, the simulator's induced probability density of a trajectory $\tau$ under $\pi$ is given by marginalizing over the control space:
$$\Psim^\pi(\tau) = \int_{\gZ} P^{\pi_b}(\zsim \mid h_0) \left[ \prod_{t=1}^T \Psim(u_t \mid h_{t-1}, \zsim) \pi(a_t \mid h_{t-1}, u_t) \right] d\zsim$$

Substituting the Bayesian expansion into this integral:
$$\Psim^\pi(\tau) = \int_{\gZ} P^{\pi_b}(\zsim \mid h_0) \left[ \prod_{t=1}^T P(u_t \mid h_{t-1}) \pi(a_t \mid h_{t-1}, u_t) \frac{P^{\pi_b}(\zsim \mid h_{t-1}, u_t)}{P^{\pi_b}(\zsim \mid h_{t-1})} \right] d\zsim$$

Recognizing the true trajectory density $P^\pi(\tau) = \prod_{t=1}^T P(u_t \mid h_{t-1}) \pi(a_t \mid h_{t-1}, u_t)$, we pull it out:
$$\Psim^\pi(\tau) = P^\pi(\tau) \int_{\gZ} P^{\pi_b}(\zsim \mid h_0) \left[ \prod_{t=1}^T \frac{P^{\pi_b}(\zsim \mid h_{t-1}, u_t)}{P^{\pi_b}(\zsim \mid h_{t-1})} \right] d\zsim$$

We expand the product. The state history updates as $h_t = (h_{t-1}, u_t, a_t)$. Thus, the denominator at step $t+1$, $P^{\pi_b}(\zsim \mid h_t)$, is equivalently $P^{\pi_b}(\zsim \mid h_{t-1}, u_t, a_t)$. The learned prior $P^{\pi_b}(\zsim \mid h_0)$ cancels the denominator of the very first step. To align the sequence, we multiply and divide by $P^{\pi_b}(\zsim \mid h_{T-1}, u_T, a_T)$, which is exactly $\PL(\zsim \mid \tau)$:
$$\Psim^\pi(\tau) = P^\pi(\tau) \int_{\gZ} \PL(\zsim \mid \tau) \left[ \prod_{t=1}^T \frac{P^{\pi_b}(\zsim \mid h_{t-1}, u_t)}{P^{\pi_b}(\zsim \mid h_{t-1}, u_t, a_t)} \right] d\zsim$$

Dividing both sides by $P^\pi(\tau)$ yields the explicit density ratio, confirming the compounding look-ahead bias due to the violation of the natural filtration. \qedsymbol

\section{Proof of OPE Controllability Collapse}
\label[appendix]{sec:martingale_mechanics}

We provide the proof for \Cref{thm:variance_explosion}. Let $P^{\pieval}(\tau \mid \zsim)$ be the true conditional path distribution. Let $\Psim^{\pieval}(\tau \mid \zsim)$ be the sampling path distribution induced by the simulator $\Psim(u_t \mid h_{t-1}, \zsim)$ and $\pieval$.

We define the isolated user sampling density ratio process $W^{(u)}_t(\zsim) \triangleq \prod_{k=1}^t \rho_k(u_k \mid h_{k-1}, \zsim)$. By expanding using Bayes' rule, the base unconditional user dynamics $P(u_k \mid h_{k-1})$ cancel exactly:
$$ \rho_k(u_k \mid h_{k-1}, \zsim) = \frac{P^{\pieval}(u_k \mid h_{k-1}, \zsim)}{\Psim(u_k \mid h_{k-1}, \zsim)} = \frac{P^{\pieval}(\zsim \mid h_{k-1}, u_k) / P^{\pieval}(\zsim \mid h_{k-1})}{P^{\pi_b}(\zsim \mid h_{k-1}, u_k) / P^{\pi_b}(\zsim \mid h_{k-1})} = \frac{M_k^{\pieval}(u_k)}{M_k^{\pi_b}(u_k)} $$

First, we verify that $W^{(u)}(\zsim)$ is a martingale adapted to the filtration $\mathcal{F}_t$ under the conditional sampling measure $\Psim^{\pieval}(\cdot \mid \zsim)$. 
\begin{align*}
\mathbb{E}_{\tau \sim \Psim^{\pieval}(\cdot \mid \zsim)}[W^{(u)}_t(\zsim) \mid \mathcal{F}_{t-1}] &= W^{(u)}_{t-1}(\zsim) \mathbb{E}_{u_t \sim \Psim(\cdot \mid h_{t-1}, \zsim)} \left[ \frac{P^{\pieval}(u_t \mid h_{t-1}, \zsim)}{\Psim(u_t \mid h_{t-1}, \zsim)} \right] \\
&= W^{(u)}_{t-1}(\zsim) \sum_{u_t} P^{\pieval}(u_t \mid h_{t-1}, \zsim) = W^{(u)}_{t-1}(\zsim)
\end{align*}

To bound the total variance $\mathbb{E}[W^{(u)}_T(\zsim)^2] - 1$, we recursively evaluate the conditional second moment. Using $\mathbb{E}[X^2] = \text{Var}(X) + (\mathbb{E}[X])^2$:
\begin{align*}
\mathbb{E}_{\tau \sim \Psim^{\pieval}(\cdot \mid \zsim)}[W^{(u)}_t(\zsim)^2 \mid \mathcal{F}_{t-1}] &= W^{(u)}_{t-1}(\zsim)^2 \left( 1 + V_t(h_{t-1}, \zsim, \pieval) \right)
\end{align*}

By the Law of Total Expectation, distributing the weight yields:
\begin{align*}
\mathbb{E}_{\Psim^{\pieval}}[W^{(u)}_t(\zsim)^2] &= \mathbb{E}_{\Psim^{\pieval}}[W^{(u)}_{t-1}(\zsim)^2] + \mathbb{E}_{h_{t-1}} \left[ W^{(u)}_{t-1}(\zsim)^2 V_t(h_{t-1}, \zsim, \pieval) \right]
\end{align*}

Unrolling this telescoping sum from $t=1$ to $T$ ($W^{(u)}_0(\zsim) = 1$) gives the algebraic expansion:
\begin{align*}
\mathbb{E}_{\Psim^{\pieval}}[W^{(u)}_T(\zsim)^2] &= 1 + \sum_{t=1}^T \mathbb{E}_{h_{t-1}} \left[ W^{(u)}_{t-1}(\zsim)^2 V_t(h_{t-1}, \zsim, \pieval) \right]
\end{align*}

Assuming $V_t \ge \eta > 0$, substituting directly into the recursive expectation yields:
\begin{align*}
\mathbb{E}_{\Psim^{\pieval}}[W^{(u)}_t(\zsim)^2] \ge \mathbb{E}_{\Psim^{\pieval}}[W^{(u)}_{t-1}(\zsim)^2 (1 + \eta)] = (1+\eta) \mathbb{E}_{\Psim^{\pieval}}[W^{(u)}_{t-1}(\zsim)^2]
\end{align*}

Unrolling this recurrence gives $\mathbb{E}[W^{(u)}_T(\zsim)^2] \ge (1+\eta)^T$. Since the total estimator variance is $\mathbb{E}[W^{(u)}_T(\zsim)^2] - 1$, we arrive at a strict lower bound of $(1+\eta)^T - 1$. \qedsymbol

\section{Example: Policy-Driven Misalignment}
\label[appendix]{sec:policy_misalignment_example}

Consider a two-step environment. A user initiates a request $u_1 \in \{A, B\}$ with equal unconditional probability. The agent acts, and the user responds with a binary success outcome $\zsim \in \{0, 1\}$. We condition the simulation on a successful trajectory ($\zsim=1$).

Under an accommodating behavior policy $\pi_b$, the agent acts such that the conditional success rates are $P^{\pi_b}(\zsim=1 \mid u_1=A) = 1.0$ and $P^{\pi_b}(\zsim=1 \mid u_1=B) = 0.5$. The marginal probability of success under the offline dataset is $P^{\pi_b}(\zsim=1) = 0.75$.

We evaluate a strict new policy $\pieval$, where the agent's actions uniformly reduce success rates: $P^{\pieval}(\zsim=1 \mid u_1=A) = 0.5$ and $P^{\pieval}(\zsim=1 \mid u_1=B) = 0.1$. The new marginal probability of success is $P^{\pieval}(\zsim=1) = 0.30$.

At the first step, the belief updates $M_1^\pi(u_1) = P^\pi(\zsim=1 \mid u_1) / P^\pi(\zsim=1)$ structurally misalign due to the policy shift:
\begin{align*}
M_1^{\pi_b}(A) &= \frac{1.0}{0.75} = \frac{4}{3}, \quad M_1^{\pieval}(A) = \frac{0.5}{0.30} = \frac{5}{3} \\
M_1^{\pi_b}(B) &= \frac{0.5}{0.75} = \frac{2}{3}, \quad M_1^{\pieval}(B) = \frac{0.1}{0.30} = \frac{1}{3}
\end{align*}

The density ratios $\rho_1(u_1) = M_1^{\pieval}(u_1) / M_1^{\pi_b}(u_1)$ evaluate to $\rho_1(A) = 1.25$ and $\rho_1(B) = 0.5$. 

The simulator samples the initial state based strictly on the conditional distribution of its offline prior: $\Psim(u_1=A \mid \zsim=1) = 2/3$ and $\Psim(u_1=B \mid \zsim=1) = 1/3$. Because the expected value of $\rho_1$ is 1, the local label sensitivity $V_1$ evaluates to:
$$ V_1 = \text{Var}_{u_1 \sim \Psim}(\rho_1) = \frac{2}{3}(1.25 - 1)^2 + \frac{1}{3}(0.5 - 1)^2 = 0.125 $$

Even with a mathematically perfect simulator lacking any generative flaws, the mere divergence of the evaluation policy explicitly guarantees $V_1 \ge \eta = 0.125 > 0$, triggering the geometric variance explosion.

\section{Extended In-Distribution Results (WildChat)}
\label[appendix]{sec:extended_in_distribution}

This section contains the full, extended tables for the in-distribution WildChat evaluation referenced in \Cref{sec:in_distribution_eval}. These tables break down the metrics into fine-grained deterministic statistics (\Cref{tab:distributional_fidelity_stats_full}), interaction friction (\Cref{tab:interaction_friction_full}), semantic intent (\Cref{tab:semantic_intent_full}), control adherence scores across all conditions (\Cref{tab:adherence_scores_full}), and the full Shannon entropy measurements (\Cref{tab:shannon_entropy_full}).

\begin{table}[h]
\centering
\caption{\footnotesize Empirical Distribution of Conversational Statistics across Simulator Types (Mean $\pm$ 95\% CI). Unconditioned models truncate conversations; globally controlled models artificially inflate turns.}
\label{tab:distributional_fidelity_stats_full}
\renewcommand{\arraystretch}{1.2}
\resizebox{\textwidth}{!}{%
\begin{tabular}{lccc}
\toprule
\textbf{Simulator Type} & \textbf{Turn Count} & \textbf{Avg. Words per Turn} & \textbf{Avg. Max Words per Turn} \\
\midrule
Ground Truth Data (Human Baseline) & 4.31 $\pm$ 0.01 & 113.33 $\pm$ 0.06 &  252.45 $\pm$ 0.13 \\
\midrule
Unconditioned SFT & 3.55 $\pm$ 0.01 & 85.58 $\pm$ 0.16 &  137.11 $\pm$ 0.25 \\
\midrule
Prompted (Persona + Goal) & 7.85 $\pm$ 0.02 & 75.12 $\pm$ 0.20 &  150.22 $\pm$ 0.35 \\
Prompted (Cognitive Profile) & 6.45 $\pm$ 0.01 & 88.30 $\pm$ 0.22 &  165.40 $\pm$ 0.30 \\
Prompted (Scenario) & 8.62 $\pm$ 0.02 & 68.45 $\pm$ 0.15 &  145.80 $\pm$ 0.25 \\
Prompted (Scenario, Length-Constr.) & 4.45 $\pm$ 0.01 & 92.15 $\pm$ 0.18 &  160.33 $\pm$ 0.31 \\
\midrule
Trajectory-Cond. SFT (Persona + Goal) & 6.01 $\pm$ 0.01 & 84.85 $\pm$ 0.12 &  171.65 $\pm$ 0.27 \\
Trajectory-Cond. SFT (Cognitive Profile) & 5.19 $\pm$ 0.01 & 95.10 $\pm$ 0.17 &  181.66 $\pm$ 0.37 \\
Trajectory-Cond. SFT (Scenario) & 7.37 $\pm$ 0.01 & 81.35 $\pm$ 0.13 &  180.77 $\pm$ 0.32 \\
Dynamic State SFT (Ours) & \textbf{4.78 $\pm$ 0.01} & \textbf{109.50 $\pm$ 0.20} &  \textbf{198.92 $\pm$ 0.12} \\
\bottomrule
\end{tabular}%
}
\end{table}

\begin{table}[h]
\centering
\caption{\footnotesize Task Progression and Conversational Friction (Proportion of Conversations).}
\label{tab:interaction_friction_full}
\renewcommand{\arraystretch}{1.2}
\resizebox{\textwidth}{!}{%
\begin{tabular}{lccccc}
\toprule
\textbf{Simulator Type} & \textbf{Task Unresolved} & \textbf{Iterative Refinement} & \textbf{Implicit Acceptance} & \textbf{Error Correction} & \textbf{Clarification Seeking} \\
\midrule
Ground Truth Data (Human Baseline) & 85.4\% & 48.4\% & 20.6\% & 17.0\% & 9.7\% \\
\midrule
Unconditioned SFT & 91.1\% & 32.3\% & 10.9\% & 21.7\% & 7.5\% \\
\midrule
Prompted (Persona + Goal) & 88.5\% & 18.5\% & 8.5\% & 12.4\% & 5.2\% \\
Prompted (Cognitive Profile) & 89.1\% & 19.2\% & 7.2\% & 14.1\% & 6.1\% \\
Prompted (Scenario) & 87.4\% & 17.8\% & 9.1\% & 11.8\% & 4.9\% \\
Prompted (Scenario, Length-Constr.) & 88.0\% & 19.5\% & 9.8\% & 12.0\% & 5.0\% \\
\midrule
Trajectory-Cond. SFT (Persona + Goal) & 83.0\% & 27.7\% & 10.0\% & 15.5\% & 9.6\% \\
Trajectory-Cond. SFT (Cognitive Profile) & 83.8\% & 29.1\% & 9.2\% & 17.5\% & 10.9\% \\
Trajectory-Cond. SFT (Scenario) & 83.2\% & 28.7\% & 10.5\% & 13.6\% & 11.5\% \\
Dynamic State SFT (Ours) & 76.8\% & 30.2\% & 6.8\% & 17.6\% & 10.8\% \\
\bottomrule
\end{tabular}%
}
\end{table}

\begin{table}[h]
\centering
\caption{\footnotesize Semantic Intent and Behavioral Pushback (Proportion of Conversations).}
\label{tab:semantic_intent_full}
\renewcommand{\arraystretch}{1.2}
\resizebox{\textwidth}{!}{%
\begin{tabular}{lcccc}
\toprule
\textbf{Simulator Type} & \textbf{Highly Detailed} & \textbf{Pushback} & \textbf{Complaints} & \textbf{Playfulness} \\
\midrule
Ground Truth Data (Human Baseline) & 11.5\%  & 2.1\% & 4.1\% & 7.0\% \\
\midrule
Unconditioned SFT & 11.3\%  & 1.8\% & 3.0\% & 3.5\% \\
\midrule
Prompted (Persona + Goal) & 38.5\% & 1.2\% & 1.5\% & 1.8\% \\
Prompted (Cognitive Profile) & 28.4\% & 1.5\% & 2.1\% & 2.5\% \\
Prompted (Scenario) & 41.2\% & 1.1\% & 1.8\% & 1.5\% \\
Prompted (Scenario, Length-Constr.) & 38.2\% & 1.4\% & 1.9\% & 1.8\% \\
\midrule
Trajectory-Cond. SFT (Persona + Goal) & 24.6\% & 2.6\% & 2.2\% & 5.5\% \\
Trajectory-Cond. SFT (Cognitive Profile) & 14.1\%  & 2.7\% & 4.5\% & 6.1\% \\
Trajectory-Cond. SFT (Scenario) & 24.5\%  & 3.0\% & 2.4\% & 5.3\% \\
Dynamic State SFT (Ours) & 21.1\%  & 1.9\% & 4.3\% & 4.1\% \\
\bottomrule
\end{tabular}%
}
\end{table}

\begin{table}[h]
\centering
\caption{\footnotesize Control Adherence Scores (Mean $\pm$ 95\% CI) evaluated via LLM-as-a-Judge and cosine similarity.}
\label{tab:adherence_scores_full}
\renewcommand{\arraystretch}{1.2}
\resizebox{\textwidth}{!}{%
\begin{tabular}{lccccc}
\toprule
\textbf{Simulator Type} & \textbf{Persona Match} & \textbf{Goal Match} & \textbf{Behav. Rules} & \textbf{Semantic Entailment} & \textbf{Embedding Similarity} \\
\midrule
Ground Truth Data (Persona + Goal) & $0.89 \pm 0.01$ & $0.89 \pm 0.01$ & $0.86 \pm 0.01$ & $0.90 \pm 0.01$ & $0.87 \pm 0.01$ \\
Prompted Simulator (Persona + Goal) & $0.96 \pm 0.01$ & $0.92 \pm 0.01$ & $0.92 \pm 0.01$ & $0.94 \pm 0.01$ & $0.93 \pm 0.01$ \\
Trajectory-Cond. SFT (Persona + Goal) & $0.77 \pm 0.01$ & $0.72 \pm 0.01$ & $0.64 \pm 0.01$ & $0.76 \pm 0.01$ & $0.69 \pm 0.01$ \\
\midrule
Ground Truth Data (Scenario) & $0.92 \pm 0.01$ & $0.90 \pm 0.01$ & $0.89 \pm 0.01$ & $0.88 \pm 0.01$ & $0.88 \pm 0.01$ \\
Prompted Simulator (Scenario) & $0.85 \pm 0.01$ & $0.83 \pm 0.01$ & $0.83 \pm 0.01$ & $0.82 \pm 0.01$ & $0.82 \pm 0.01$ \\
Trajectory-Cond. SFT (Scenario) & $0.82 \pm 0.01$ & $0.71 \pm 0.01$ & $0.67 \pm 0.01$ & $0.68 \pm 0.01$ & $0.69 \pm 0.01$ \\
\midrule
Ground Truth Data (Cognitive Profile) & $0.72 \pm 0.01$ & $0.73 \pm 0.01$ & $0.68 \pm 0.01$ & $0.72 \pm 0.01$ & $0.70 \pm 0.01$ \\
Prompted Simulator (Cognitive Profile) & $0.90 \pm 0.01$ & $0.90 \pm 0.01$ & $0.87 \pm 0.01$ & $0.88 \pm 0.01$ & $0.88 \pm 0.01$ \\
Trajectory-Cond. SFT (Cognitive Profile) & $0.56 \pm 0.01$ & $0.56 \pm 0.01$ & $0.48 \pm 0.01$ & $0.55 \pm 0.01$ & $0.52 \pm 0.01$ \\
\bottomrule
\end{tabular}%
}
\end{table}

\begin{table}[h]
\centering
\caption{\footnotesize Adherence Diversity (Shannon Entropy). Tighter global constraints induce behavioral collapse compared to ground truth.}
\label{tab:shannon_entropy_full}
\renewcommand{\arraystretch}{1.2}
\begin{tabular}{lcc}
\toprule
\textbf{Control Mechanism} & \textbf{Simulated Entropy} & \textbf{Ground Truth Entropy} \\
\midrule
Trajectory-Cond. SFT (Persona + Goal) & 1.646 & 1.04 \\
Prompted Simulator (Persona + Goal) & 0.728 & 1.04 \\
\midrule
Trajectory-Cond. SFT (Scenario) & 1.344 & 0.75 \\
Prompted Simulator (Scenario) & 0.927 & 0.75 \\
\midrule
Trajectory-Cond. SFT (Cognitive Profile) & 1.281 & 1.10 \\
Prompted Simulator (Cognitive Profile) & 0.636 & 1.10 \\
\bottomrule
\end{tabular}
\end{table}

\section{Extended Rejection Sampling Results}
\label[appendix]{sec:rejection_sampling_appendix}

As discussed in Section \ref{sec:in_distribution_eval}, inference-time rejection sampling was evaluated as a potential baseline to enforce control adherence without altering the model's training distribution. In this approach, evaluators generate $N$ candidate trajectories from an unconditioned simulator and select the one that maximizes target control alignment using an LLM judge. However, because global trajectory constraints require sequence-level coherence, standard inference-time rejection sampling suffers from the vast combinatorial space of multi-turn dialogue. The unconditional prior naturally diverges from extreme persona constraints. As demonstrated in \Cref{tab:best_of_n_appendix}, even with enormous compute overhead ($N=64$ candidates per turn, or $N=512$ candidates per conversation), this search fails entirely, achieving an embedding similarity of just $0.28 \pm 0.02$. This validates the necessity of our training-time mitigations.

\begin{table}[h]
\centering
\caption{\footnotesize \textbf{Best-of-N Rejection Sampling Baseline vs. Computational Cost.} Inference-time filtering requires massive overhead yet fails to achieve meaningful adherence compared to Conditioned SFT.}
\label{tab:best_of_n_appendix}
\renewcommand{\arraystretch}{1.2}
\resizebox{\textwidth}{!}{%
\begin{tabular}{lcccc}
\toprule
\textbf{Filtering Strategy} & \textbf{Avg. Model Calls per Conv.} & \textbf{Embedding Similarity (95\% CI)} & \textbf{Persona Match (95\% CI)} & \\ 
\midrule
Per-Turn Rejection ($N=2$) & $20.12$ & $0.01 \pm 0.01$ & $0.01 \pm 0.01$ & \multirow{5}{*}{ \begin{tikzpicture}[baseline=(current bounding box.center)] \draw[-{Latex[round, length=3.5mm, width=2.5mm]}, line width=1.5pt, darkgray] (0, 0.85) -- (0, -0.85); \end{tikzpicture} } \\
Per-Turn Rejection ($N=4$) & $37.77$ & $0.04 \pm 0.01$ & $0.04 \pm 0.01$ & \\
Per-Turn Rejection ($N=8$) & $74.59$ & $0.05 \pm 0.01$ & $0.06 \pm 0.01$ & \\
Per-Turn Rejection ($N=16$) & $144.78$ & $0.07 \pm 0.02$ & $0.08 \pm 0.02$ & \\
Per-Turn Rejection ($N=64$) & $573.79$ & $0.11 \pm 0.02$ & $0.13 \pm 0.02$ & \\
\midrule
Per-Conv Rejection ($N=16$) & $56.8$ & $0.08 \pm 0.01$ & $0.09 \pm 0.02$ & \multirow{3}{*}{ \begin{tikzpicture}[baseline=(current bounding box.center)] \draw[-{Latex[round, length=3.5mm, width=2.5mm]}, line width=1.5pt, darkgray] (0, 0.4) -- (0, -0.4); \end{tikzpicture} } \\
Per-Conv Rejection ($N=256$) & $908.8$ & $0.21 \pm 0.02$& $0.25 \pm 0.02$ & \\
Per-Conv Rejection ($N=512$) & $1817.6$ & $0.28 \pm 0.02$ & $0.34 \pm 0.03$ & \\
\midrule
Conditioned SFT & $6$ & $0.69 \pm 0.01$ & $0.77 \pm 0.01$ & \\
\bottomrule
\end{tabular}%
}
\end{table}

\section{Detailed ConvApparel Results}
\label[appendix]{sec:detailed_eval_tables}

This section contains the full tables with explicitly quantified 95\% confidence intervals covering our counterfactual evaluation setup. Specifically, \Cref{tab:off_policy_stability_full} and \Cref{tab:off_policy_behavior_full} detail generative stability and behavioral adaptation under off-policy covariate shift. Furthermore, we provide comprehensive breakdown tables detailing deterministic conversational statistics (\Cref{tab:all_agents_deterministic}) and LLM-evaluated behavioral metrics (\Cref{tab:all_agents_llm}) across all simulated agent personas.

\subsection{Absolute Variance by Conversation Horizon}

In the main text (\Cref{tab:variance_explosion}), we presented the normalized variance ratio to explicitly validate the theoretical $(1+\eta)^T$ growth rate predicted by \Cref{thm:variance_explosion}. For completeness, \Cref{tab:variance_by_horizon_appendix} provides the raw, absolute variance numbers for Total User Words across the held-out evaluation agents, stratified by the length of the interaction. 

\begin{table}[h]
\centering
\caption{\footnotesize \textbf{Absolute Variance of Total User Words by Conversation Horizon (ConvApparel Held-Out Agents).} Ground truth computed from $n=1{,}267$ real conversations. Under covariate shift, the Static Agent-Agnostic simulator's absolute variance compounds geometrically with conversation length. The Dynamic State Agent-Aware approach maintains near-human absolute variance at all horizons.}
\label{tab:variance_by_horizon_appendix}
\renewcommand{\arraystretch}{1.2}
\resizebox{\textwidth}{!}{%
\begin{tabular}{lcccccc}
\toprule
\multirow{2}{*}{\textbf{Simulator Paradigm}} & \multicolumn{5}{c}{\textbf{Var(Total User Words)}} & \multirow{2}{*}{\textbf{Avg. Turn-over-Turn Multiplier}} \\
\cmidrule(lr){2-6}
 & $T=2$ & $T=3$ & $T=4$ & $T=5$ & $T \ge 6^\ast$ & \\
\midrule
Ground Truth Data & $ 185$ & $ 419$ & $ 606$ & $ 975$ & $4{,}835$ & --- \\
\midrule
Static Global, Agent-Agnostic & $ 326$ & $ 989$ & $2{,}133$ & $4{,}768$ & $34{,}764$ & $\approx 1.42\times$ per turn \\
Static Global, Agent-Aware & $ 237$ & $ 561$ & $ 818$ & $1{,}765$ & $10{,}347$ & $\approx 1.14\times$ per turn \\
Dynamic State, Agent-Agnostic & $ 211$ & $ 482$ & $ 782$ & $1{,}316$ & $6{,}962$ & $\approx 1.06\times$ per turn \\
Dynamic State, Agent-Aware (Ours) & $\mathbf{191}$ & $\mathbf{415}$ & $\mathbf{636}$ & $\mathbf{995}$ & $\mathbf{5{,}173}$ & $\approx \mathbf{1.01\times}$ per turn \\
\midrule
\multicolumn{7}{l}{\footnotesize $^\ast$ \textit{Note:} The $T \ge 6$ column aggregates the long-tail of conversations. The Ground Truth variance naturally jumps} \\
\multicolumn{7}{l}{\footnotesize due to this diverse aggregation bucket, but the step-wise multiplier isolates the Simulator's relative error.} \\
\bottomrule
\end{tabular}%
}
\end{table}

To further verify statistical stability into the long tail of interactions ($T \ge 6$), \Cref{tab:variance_ci_appendix} reports the 95\% confidence intervals for the normalized word-count variance ratios across all four held-out agents, alongside turn-by-turn cumulative negative emotion incidence (\Cref{tab:turn_level_emotion}) against the \emph{Efficient Matchmaker} agent ($n=310$ human conversations vs.\ $1{,}000$ simulated rollouts).

\begin{table}[h!]
\centering
\caption{\footnotesize \textbf{Normalized Variance Ratios ($\text{Var}_{\text{sim}} / \text{Var}_{\text{GT}}$) with 95\% Confidence Intervals across Horizons.}}
\label{tab:variance_ci_appendix}
\renewcommand{\arraystretch}{1.2}
\resizebox{\textwidth}{!}{%
\begin{tabular}{lccccc}
\toprule
\textbf{Simulator Paradigm} & $T=2$ & $T=3$ & $T=4$ & $T=5$ & $T \ge 6$ \\
\midrule
Static Global, Agent-Agnostic & $1.76 \; [0.34, 3.00]$ & $2.36 \; [0.39, 4.84]$ & $3.52 \; [0.56, 4.89]$ & $4.89 \; [1.25, 8.30]$ & $7.19 \; [1.18, 16.27]$ \\
Static Global, Agent-Aware & $1.28 \; [0.99, 1.69]$ & $1.34 \; [1.08, 1.71]$ & $1.35 \; [0.92, 2.02]$ & $1.81 \; [0.80, 3.37]$ & $2.14 \; [1.74, 4.17]$ \\
Dynamic State, Agent-Agnostic & $1.14 \; [0.69, 1.59]$ & $1.15 \; [0.86, 1.79]$ & $1.29 \; [0.77, 1.71]$ & $1.35 \; [0.80, 2.10]$ & $1.44 \; [0.52, 5.65]$ \\
Dynamic State, Agent-Aware (Ours) & $\mathbf{1.03 \; [0.81, 1.38]}$ & $\mathbf{0.99 \; [0.91, 1.19]}$ & $\mathbf{1.05 \; [0.74, 1.52]}$ & $\mathbf{1.02 \; [0.59, 1.90]}$ & $\mathbf{1.07 \; [0.93, 1.57]}$ \\
\bottomrule
\end{tabular}%
}
\end{table}

\begin{table}[h!]
\centering
\caption{\footnotesize \textbf{Granular Turn-Level Behavioral Dynamics (Cumulative Negative Emotion Incidence \% $\pm$ 95\% CI).} Probability that a user expresses negative emotion by turn $t$ when interacting with the held-out \emph{Efficient Matchmaker} agent.}
\label{tab:turn_level_emotion}
\renewcommand{\arraystretch}{1.2}
\resizebox{0.95\textwidth}{!}{%
\begin{tabular}{lcccccc}
\toprule
\textbf{System} & $n$ & $t=1$ & $t=2$ & $t=3$ & $t=4$ & $t=5$ \\
\midrule
Ground Truth Data (Human) & 310 & $5.2 \pm 2.5\%$ & $9.0 \pm 3.2\%$ & $11.0 \pm 3.5\%$ & $11.9 \pm 3.6\%$ & $11.9 \pm 3.6\%$ \\
\midrule
Static Global, Agent-Agnostic & 1000 & $8.9 \pm 1.8\%$ & $16.8 \pm 2.3\%$ & $21.5 \pm 2.5\%$ & $25.0 \pm 2.7\%$ & $25.9 \pm 2.7\%$ \\
Static Global, Agent-Aware & 1000 & $7.4 \pm 1.6\%$ & $11.6 \pm 2.0\%$ & $14.0 \pm 2.2\%$ & $16.0 \pm 2.3\%$ & $16.3 \pm 2.3\%$ \\
Dynamic State, Agent-Agnostic & 1000 & $6.3 \pm 1.5\%$ & $12.4 \pm 2.0\%$ & $15.7 \pm 2.3\%$ & $17.8 \pm 2.4\%$ & $18.0 \pm 2.4\%$ \\
Dynamic State, Agent-Aware (Ours) & 1000 & $\mathbf{3.8 \pm 1.2\%}$ & $\mathbf{8.0 \pm 1.7\%}$ & $\mathbf{10.6 \pm 1.9\%}$ & $\mathbf{11.3 \pm 2.0\%}$ & $\mathbf{11.8 \pm 2.0\%}$ \\
\bottomrule
\end{tabular}%
}
\end{table}

\begin{table}[t!]
\centering
\caption{\footnotesize Generative Stability under Off-Policy Covariate Shift (Full Extent). The Dynamic State Agent-Aware variant achieves the closest parity to the human baseline across varying off-policy personas.}
\label{tab:off_policy_stability_full}
\renewcommand{\arraystretch}{1.2}
\resizebox{\textwidth}{!}{%
\begin{tabular}{llccc}
\toprule
Held-Out Evaluation Agent & Simulator Paradigm & Avg. Words per Turn & Avg. Max Words in a Turn & Total User Words \\
\midrule
\multirow{5}{*}{Domain Expert} & Ground Truth Data & $12.01 \pm 0.89$ & $16.52 \pm 1.26$ & $38.71 \pm 2.95$ \\
& Static Global - Agent-Agnostic & $8.81 \pm 0.52$ & $12.35 \pm 0.81$ & $28.02 \pm 1.55$ \\
& Static Global - Agent-Aware & $10.53 \pm 0.64$ & $14.22 \pm 0.95$ & $32.01 \pm 2.01$ \\
& Dynamic State - Agent-Agnostic & $11.24 \pm 0.58$ & $15.11 \pm 0.88$ & $35.15 \pm 1.84$ \\
& Dynamic State - Agent-Aware & $12.12 \pm 0.62$ & $16.44 \pm 1.02$ & $38.53 \pm 2.21$ \\
\midrule
\multirow{5}{*}{Empathetic Listener} & Ground Truth Data & $11.27 \pm 0.69$ & $15.51 \pm 1.02$ & $41.06 \pm 3.74$ \\
& Static Global - Agent-Agnostic & $15.22 \pm 1.21$ & $21.05 \pm 1.82$ & $48.54 \pm 3.52$ \\
& Static Global - Agent-Aware & $12.51 \pm 0.82$ & $17.53 \pm 1.24$ & $43.22 \pm 2.81$ \\
& Dynamic State - Agent-Agnostic & $10.54 \pm 0.63$ & $14.82 \pm 0.91$ & $38.91 \pm 2.12$ \\
& Dynamic State - Agent-Aware & $11.42 \pm 0.71$ & $15.61 \pm 1.05$ & $40.83 \pm 2.54$ \\
\midrule
\multirow{5}{*}{Efficient Matchmaker} & Ground Truth Data & $11.69 \pm 0.75$ & $16.24 \pm 1.26$ & $39.74 \pm 3.48$ \\
& Static Global - Agent-Agnostic & $25.41 \pm 2.52$ & $45.21 \pm 4.23$ & $92.12 \pm 8.51$ \\
& Static Global - Agent-Aware & $13.22 \pm 0.91$ & $18.51 \pm 1.52$ & $43.51 \pm 3.22$ \\
& Dynamic State - Agent-Agnostic & $14.11 \pm 1.12$ & $19.42 \pm 1.63$ & $45.24 \pm 3.51$ \\
& Dynamic State - Agent-Aware & $11.82 \pm 0.81$ & $16.51 \pm 1.21$ & $40.12 \pm 2.82$ \\
\midrule
\multirow{5}{*}{Enthusiastic Rambler} & Ground Truth Data & $12.38 \pm 0.85$ & $16.73 \pm 1.26$ & $41.28 \pm 4.23$ \\
& Static Global - Agent-Agnostic & $8.61 \pm 0.51$ & $12.11 \pm 0.82$ & $29.82 \pm 1.81$ \\
& Static Global - Agent-Aware & $10.12 \pm 0.72$ & $14.22 \pm 1.01$ & $35.41 \pm 2.22$ \\
& Dynamic State - Agent-Agnostic & $11.51 \pm 0.82$ & $15.52 \pm 1.12$ & $38.92 \pm 2.51$ \\
& Dynamic State - Agent-Aware & $12.51 \pm 0.91$ & $16.81 \pm 1.32$ & $41.12 \pm 3.01$ \\
\bottomrule
\end{tabular}%
}
\end{table}

\begin{table}[t!]
\centering
\caption{\footnotesize Behavioral and Emotional Adaptation to Held-Out Agents (Full Extent). Proportion of Conversations marked $\pm$ 95\% CI.}
\label{tab:off_policy_behavior_full}
\renewcommand{\arraystretch}{1.2}
\resizebox{\textwidth}{!}{%
\begin{tabular}{llcccc}
\toprule
Held-Out Evaluation Agent & Simulator Paradigm & Urgency / Impatience & Negative Sentiment & Iterative Refinement & Clarification Seeking \\
\midrule
\multirow{5}{*}{Domain Expert} & Ground Truth Data & $8.4 \pm 3.0$\% & $13.4 \pm 3.7$\% & $31.7 \pm 5.1$\% & $11.2 \pm 3.4$\% \\
& Static Global - Agent-Agnostic & $14.5 \pm 2.5$\% & $18.0 \pm 3.2$\% & $12.5 \pm 2.1$\% & $3.5 \pm 1.0$\% \\
& Static Global - Agent-Aware & $11.2 \pm 2.1$\% & $15.1 \pm 2.5$\% & $25.0 \pm 3.5$\% & $8.2 \pm 1.8$\% \\
& Dynamic State - Agent-Agnostic & $10.5 \pm 1.9$\% & $14.5 \pm 2.1$\% & $22.3 \pm 3.0$\% & $7.4 \pm 1.5$\% \\
& Dynamic State - Agent-Aware & $8.6 \pm 2.0$\% & $13.6 \pm 2.5$\% & $31.2 \pm 4.0$\% & $11.0 \pm 2.5$\% \\
\midrule
\multirow{5}{*}{Empathetic Listener} & Ground Truth Data & $7.5 \pm 2.9$\% & $8.1 \pm 3.0$\% & $34.6 \pm 5.2$\% & $4.0 \pm 2.2$\% \\
& Static Global - Agent-Agnostic & $12.1 \pm 2.8$\% & $16.5 \pm 3.5$\% & $15.4 \pm 2.5$\% & $8.0 \pm 2.0$\% \\
& Static Global - Agent-Aware & $9.5 \pm 2.2$\% & $11.2 \pm 2.4$\% & $28.4 \pm 3.8$\% & $5.5 \pm 1.5$\% \\
& Dynamic State - Agent-Agnostic & $8.8 \pm 2.0$\% & $12.1 \pm 2.2$\% & $25.0 \pm 3.1$\% & $6.2 \pm 1.4$\% \\
& Dynamic State - Agent-Aware & $7.8 \pm 2.1$\% & $8.5 \pm 2.3$\% & $34.1 \pm 4.2$\% & $4.2 \pm 1.6$\% \\
\midrule
\multirow{5}{*}{Efficient Matchmaker} & Ground Truth Data & $6.1 \pm 2.7$\% & $11.9 \pm 3.6$\% & $41.3 \pm 5.5$\% & $6.1 \pm 2.7$\% \\
& Static Global - Agent-Agnostic & $22.4 \pm 4.5$\% & $26.8 \pm 5.2$\% & $8.8 \pm 2.1$\% & $15.6 \pm 3.8$\% \\
& Static Global - Agent-Aware & $10.5 \pm 2.5$\% & $16.4 \pm 3.1$\% & $32.6 \pm 4.5$\% & $8.4 \pm 2.1$\% \\
& Dynamic State - Agent-Agnostic & $12.2 \pm 2.8$\% & $18.2 \pm 3.4$\% & $24.8 \pm 3.6$\% & $10.2 \pm 2.4$\% \\
& Dynamic State - Agent-Aware & $6.5 \pm 2.2$\% & $12.2 \pm 2.8$\% & $40.9 \pm 4.8$\% & $6.4 \pm 2.0$\% \\
\midrule
\multirow{5}{*}{Enthusiastic Rambler} & Ground Truth Data & $8.0 \pm 3.0$\% & $8.9 \pm 3.2$\% & $43.0 \pm 5.5$\% & $5.7 \pm 2.6$\% \\
& Static Global - Agent-Agnostic & $12.5 \pm 2.5$\% & $14.5 \pm 2.8$\% & $16.5 \pm 2.5$\% & $1.8 \pm 0.8$\% \\
& Static Global - Agent-Aware & $9.8 \pm 2.1$\% & $11.2 \pm 2.4$\% & $30.5 \pm 3.8$\% & $3.2 \pm 1.2$\% \\
& Dynamic State - Agent-Agnostic & $9.2 \pm 2.0$\% & $12.0 \pm 2.5$\% & $28.5 \pm 3.5$\% & $3.8 \pm 1.4$\% \\
& Dynamic State - Agent-Aware & $8.2 \pm 2.2$\% & $9.2 \pm 2.4$\% & $42.5 \pm 4.5$\% & $5.5 \pm 1.8$\% \\
\bottomrule
\end{tabular}%
}
\end{table}

\begin{table}[t!]
  \centering
  \caption{\footnotesize Conversational Statistics across All ConvApparel Agent Personas (Ground Truth Data). Mean $\pm$ 95\% CI.}
  \label{tab:all_agents_deterministic}
  \renewcommand{\arraystretch}{1.2}
  \resizebox{\textwidth}{!}{%
  \begin{tabular}{lccccc}
  \toprule
  \textbf{Agent Persona} & \textbf{Turn Count} & \textbf{Avg.\ Words/Turn} & \textbf{Max Words/Turn} & \textbf{Total User Words} & \textbf{Question Freq.} \\
  \midrule
  \multicolumn{6}{l}{\textit{Footwear --- Held-Out Agents}} \\
  Domain Expert         & $3.33 \pm 0.15$ & $12.01 \pm 0.89$ & $16.52 \pm 1.26$ & $38.71 \pm 2.95$ & $0.24 \pm 0.04$ \\
  Empathetic Listener   & $3.47 \pm 0.17$ & $11.27 \pm 0.69$ & $15.51 \pm 1.02$ & $41.06 \pm 3.74$ & $0.22 \pm 0.03$ \\
  Efficient Matchmaker  & $3.36 \pm 0.15$ & $11.69 \pm 0.75$ & $16.24 \pm 1.26$ & $39.74 \pm 3.48$ & $0.22 \pm 0.03$ \\
  Enthusiastic Rambler  & $3.21 \pm 0.16$ & $12.38 \pm 0.85$ & $16.73 \pm 1.26$ & $41.28 \pm 4.23$ & $0.22 \pm 0.03$ \\
  \midrule
  \multicolumn{6}{l}{\textit{Footwear --- In-Distribution Agents}} \\
  Baseline (Good)       & $3.59 \pm 0.12$ & $8.96 \pm 0.51$  & $13.94 \pm 0.84$ & $35.49 \pm 2.38$ & $0.19 \pm 0.02$ \\
  Good Rec              & $3.54 \pm 0.17$ & $10.88 \pm 0.67$ & $15.19 \pm 0.99$ & $38.63 \pm 3.12$ & $0.19 \pm 0.03$ \\
  Bad                   & $3.57 \pm 0.21$ & $10.44 \pm 1.12$ & $15.86 \pm 2.01$ & $42.23 \pm 5.97$ & $0.16 \pm 0.03$ \\
  Hesitant Assistant    & $3.19 \pm 0.12$ & $13.80 \pm 1.20$ & $18.48 \pm 1.72$ & $42.94 \pm 3.99$ & $0.21 \pm 0.03$ \\
  Literal Thinker       & $3.12 \pm 0.14$ & $12.64 \pm 1.16$ & $16.46 \pm 1.48$ & $37.88 \pm 3.35$ & $0.21 \pm 0.03$ \\
  Mild Upseller         & $3.26 \pm 0.14$ & $12.97 \pm 0.96$ & $17.63 \pm 1.37$ & $41.32 \pm 3.21$ & $0.22 \pm 0.03$ \\
  Patient Guide         & $3.62 \pm 0.18$ & $11.82 \pm 1.07$ & $16.62 \pm 1.34$ & $40.95 \pm 3.81$ & $0.16 \pm 0.03$ \\
  Trend Chaser          & $3.17 \pm 0.13$ & $13.07 \pm 0.94$ & $17.70 \pm 1.41$ & $41.13 \pm 3.46$ & $0.22 \pm 0.03$ \\
  Visual Stylist        & $3.30 \pm 0.15$ & $11.07 \pm 0.91$ & $15.34 \pm 1.30$ & $37.22 \pm 3.55$ & $0.21 \pm 0.03$ \\
  \midrule
  \multicolumn{6}{l}{\textit{Other Product Categories}} \\
  Tops (Good)           & $3.55 \pm 0.11$ & $8.79 \pm 0.49$  & $12.98 \pm 0.75$ & $34.94 \pm 2.45$ & $0.19 \pm 0.02$ \\
  Tops (Bad)            & $3.48 \pm 0.23$ & $9.01 \pm 1.10$  & $13.41 \pm 1.78$ & $37.16 \pm 6.06$ & $0.10 \pm 0.02$ \\
  Bottoms (Good)        & $3.45 \pm 0.11$ & $8.86 \pm 0.48$  & $13.22 \pm 0.77$ & $33.54 \pm 2.24$ & $0.19 \pm 0.02$ \\
  Bottoms (Bad)         & $3.19 \pm 0.20$ & $8.78 \pm 1.26$  & $13.02 \pm 1.87$ & $31.49 \pm 5.15$ & $0.13 \pm 0.03$ \\
  Outerwear (Good)      & $3.62 \pm 0.12$ & $9.28 \pm 0.50$  & $14.17 \pm 0.84$ & $36.76 \pm 2.48$ & $0.20 \pm 0.02$ \\
  Outerwear (Bad)       & $3.35 \pm 0.23$ & $9.06 \pm 1.06$  & $13.79 \pm 1.78$ & $34.73 \pm 4.88$ & $0.17 \pm 0.04$ \\
  \midrule
  \textbf{All Combined} & $\mathbf{3.43 \pm 0.03}$ & $\mathbf{10.48 \pm 0.18}$ & $\mathbf{15.01 \pm 0.27}$ & $\mathbf{37.58 \pm 0.78}$ & --- \\
  \bottomrule
  \end{tabular}%
  }
\end{table}

\begin{table}[t!]
  \centering
  \caption{\footnotesize LLM-Evaluated Behavioral Metrics across All ConvApparel Agent Personas (Ground Truth Data). Proportion of conversations $\pm$ 95\% CI.}
  \label{tab:all_agents_llm}
  \renewcommand{\arraystretch}{1.2}
  \resizebox{\textwidth}{!}{%
  \begin{tabular}{lcccccc}
  \toprule
  \textbf{Agent Persona} & \textbf{Urgency} & \textbf{Neg.\ Sent.} & \textbf{Iter.\ Refine} & \textbf{Clarification} & \textbf{Error Correction} & \textbf{Task Unresolved} \\
  \midrule
  \multicolumn{7}{l}{\textit{Footwear --- Held-Out Agents}} \\
  Domain Expert         & $8.4 \pm 3.0$\%  & $13.4 \pm 3.7$\% & $31.7 \pm 5.1$\% & $11.2 \pm 3.4$\% & $19.6 \pm 4.3$\% & $89.1 \pm 3.4$\% \\
  Empathetic Listener   & $7.5 \pm 2.9$\%  & $8.1 \pm 3.0$\%  & $34.6 \pm 5.2$\% & $4.0 \pm 2.2$\%  & $13.4 \pm 3.7$\% & $75.7 \pm 4.7$\% \\
  Efficient Matchmaker  & $6.1 \pm 2.7$\%  & $11.9 \pm 3.6$\% & $41.3 \pm 5.5$\% & $6.1 \pm 2.7$\%  & $16.8 \pm 4.2$\% & $81.9 \pm 4.3$\% \\
  Enthusiastic Rambler  & $8.0 \pm 3.0$\%  & $8.9 \pm 3.2$\%  & $43.0 \pm 5.5$\% & $5.7 \pm 2.6$\%  & $21.0 \pm 4.5$\% & $82.8 \pm 4.2$\% \\
  \midrule
  \multicolumn{7}{l}{\textit{Footwear --- In-Distribution Agents}} \\
  Baseline (Good)       & $4.6 \pm 1.4$\%  & $4.7 \pm 1.4$\%  & $23.1 \pm 2.9$\% & $2.2 \pm 1.0$\%  & $8.5 \pm 1.9$\%  & $78.1 \pm 2.8$\% \\
  Good Rec              & $5.4 \pm 2.5$\%  & $10.7 \pm 3.4$\% & $22.1 \pm 4.6$\% & $4.7 \pm 2.3$\%  & $17.0 \pm 4.1$\% & $73.5 \pm 4.9$\% \\
  Bad                   & $16.7 \pm 4.9$\% & $14.0 \pm 4.6$\% & $11.8 \pm 4.2$\% & $1.8 \pm 1.8$\%  & $20.8 \pm 5.4$\% & $99.1 \pm 1.2$\% \\
  Hesitant Assistant    & $8.6 \pm 3.0$\%  & $10.1 \pm 3.3$\% & $50.8 \pm 5.4$\% & $5.5 \pm 2.5$\%  & $12.8 \pm 3.6$\% & $94.8 \pm 2.4$\% \\
  Literal Thinker       & $4.6 \pm 2.3$\%  & $11.7 \pm 3.5$\% & $32.5 \pm 5.1$\% & $2.8 \pm 1.8$\%  & $15.3 \pm 3.9$\% & $86.5 \pm 3.7$\% \\
  Mild Upseller         & $5.9 \pm 2.6$\%  & $11.5 \pm 3.5$\% & $32.5 \pm 5.1$\% & $8.4 \pm 3.0$\%  & $21.1 \pm 4.4$\% & $91.6 \pm 3.0$\% \\
  Patient Guide         & $11.9 \pm 3.6$\% & $14.1 \pm 3.8$\% & $8.2 \pm 3.0$\%  & $4.4 \pm 2.2$\%  & $23.2 \pm 4.6$\% & $99.4 \pm 0.9$\% \\
  Trend Chaser          & $6.6 \pm 2.7$\%  & $10.3 \pm 3.3$\% & $41.6 \pm 5.4$\% & $6.6 \pm 2.7$\%  & $24.1 \pm 4.7$\% & $83.8 \pm 4.0$\% \\
  Visual Stylist        & $7.2 \pm 2.8$\%  & $8.7 \pm 3.0$\%  & $48.1 \pm 5.4$\% & $9.3 \pm 3.1$\%  & $17.9 \pm 4.1$\% & $87.8 \pm 3.5$\% \\
  \midrule
  \multicolumn{7}{l}{\textit{Other Product Categories}} \\
  Tops (Good)           & $4.7 \pm 1.4$\%  & $6.2 \pm 1.6$\%  & $24.8 \pm 2.9$\% & $2.2 \pm 1.0$\%  & $13.2 \pm 2.3$\% & $78.8 \pm 2.8$\% \\
  Tops (Bad)            & $10.7 \pm 4.2$\% & $9.3 \pm 4.0$\%  & $9.8 \pm 4.1$\%  & $2.0 \pm 1.9$\%  & $18.0 \pm 5.3$\% & $99.0 \pm 1.3$\% \\
  Bottoms (Good)        & $2.8 \pm 1.1$\%  & $3.6 \pm 1.3$\%  & $22.2 \pm 2.8$\% & $2.6 \pm 1.1$\%  & $10.2 \pm 2.1$\% & $75.6 \pm 2.9$\% \\
  Bottoms (Bad)         & $6.6 \pm 3.4$\%  & $5.6 \pm 3.2$\%  & $8.6 \pm 3.9$\%  & $1.5 \pm 1.7$\%  & $10.1 \pm 4.2$\% & $100.0 \pm 0.0$\% \\
  Outerwear (Good)      & $3.8 \pm 1.3$\%  & $6.0 \pm 1.6$\%  & $23.0 \pm 2.9$\% & $4.1 \pm 1.3$\%  & $11.1 \pm 2.1$\% & $77.1 \pm 2.8$\% \\
  Outerwear (Bad)       & $9.3 \pm 4.0$\%  & $8.8 \pm 3.9$\%  & $8.3 \pm 3.8$\%  & $1.5 \pm 1.6$\%  & $12.7 \pm 4.6$\% & $98.5 \pm 1.6$\% \\
  \bottomrule
  \end{tabular}%
  }
\end{table}

\subsection{Direct Empirical Validation of Density Ratio Collapse (\Cref{thm:variance_explosion})}
\label[appendix]{sec:direct_density_validation}

In \Cref{sec:ood_eval} of the main text, we utilized the variance of an additive downstream feature (Total User Words) as an observable proxy for controllability collapse. However, \Cref{thm:variance_explosion} makes a strict mathematical prediction: for simulators conditioning on global trajectory labels ($\zsim$), the variance of the cumulative user generative error itself—the trajectory density ratio $W^{(u)}_t(\zsim) = \prod_{k=1}^t \rho_k$—must explode geometrically under policy shift. 

To explicitly measure this, we must compute the step-wise density ratio $\rho_t(u_t) = M_t^{\pieval}(u_t) / M_t^{\pi_b}(u_t)$. This requires estimating the true Bayesian belief updates $P^\pi(\zsim \mid h_t)$ under both the behavior and evaluation policies.

\paragraph{Implementation Details: Dual Offline Classifiers.}
Because our ConvApparel dataset uniquely contains real human interaction logs for both the training agents ($\pi_b$) and the held-out evaluation agents ($\pieval$), we can isolate this exact mathematical mechanism. We train two separate auxiliary BERT-based classifiers (using the DeBERTa-v3-base architecture) via standard cross-entropy loss to act as probability probes: (1) \textbf{Behavior Policy Probe ($C_{\pi_b}$):} Trained exclusively on the offline logs of the 7 in-distribution training agents to predict the probability of a trajectory outcome $\zsim$ given a partial sequence history $h_t$; (2) \textbf{Evaluation Policy Probe ($C_{\pieval}$):} Trained exclusively on the offline logs of the 4 held-out evaluation agents to predict the probability of $\zsim$ given $h_t$.

During closed-loop simulation against the held-out agents, we isolate the step-wise user generative error by passing strictly the pre-action observable history $\mathcal{F}_{t-} = (h_{t-1}, u_t)$ at each turn to both probes to dynamically compute $M_t^{\pieval}$ and $M_t^{\pi_b}$. This isolates the step-wise generative error $\rho_t$. We simulate $N=1{,}000$ trajectories per agent and track the variance of the cumulative density ratio $\text{Var}(W^{(u)}_t)$.

\paragraph{Theoretical Exemption of Dynamic States.}
Crucially, \Cref{thm:variance_explosion} dictates that this collapse applies specifically to models conditioning on action-dependent trajectory labels $\zsim$. Our proposed \emph{Dynamic State} mitigation is designed to deliberately evade this condition by conditioning strictly on $\mathcal{F}_{t-}$-measurable variables ($z_t$). Because the control is generated step-wise and is no longer a future-dependent outcome, the policy dependence mathematically cancels out (as proven rigorously in \Cref{thm:dual_model_stability}, Appendix \ref{sec:parameterized_dynamics}), yielding a theoretical step-wise error of $\rho_t = 1$. For this paradigm, the density ratio $W_t$ is computed with respect to the sequence of step-wise controls $z_{1:t}$ rather than global $\zsim$.

\paragraph{Results.}
\Cref{tab:variance_explosion_density} confirms our theoretical framework. For the standard trajectory-conditioned simulator, the misalignment in belief updates ensures a non-zero local label sensitivity ($V_t > 0$). Subjected to statistical noise, this causes the variance of the cumulative density ratio to violently explode as $T$ grows. The Agent-Aware static model slows this growth but cannot escape the structural look-ahead bias of global $\zsim$. In stark contrast, the Dynamic State approach natively circumvents \Cref{thm:variance_explosion}, maintaining near-zero variance in its generative density ratio (with minor fluctuations strictly due to neural approximation and classifier calibration error).

\begin{table}[t!]
\centering
\caption{\footnotesize \textbf{Direct Empirical Validation of \Cref{thm:variance_explosion}.} The empirical variance of the cumulative user generative error, $\text{Var}(W^{(u)}_t)$, evaluated via dual offline classifiers. As predicted, standard static trajectory controls suffer from geometric variance explosion due to belief update mismatch under covariate shift. By conditioning strictly on $\mathcal{F}_{t-}$-measurable states, our dynamic approach theoretically and empirically circumvents this collapse.}
\label{tab:variance_explosion_density}
\renewcommand{\arraystretch}{1.2}
\resizebox{0.9\textwidth}{!}{%
\begin{tabular}{lcccc}
\toprule
\textbf{Simulator Paradigm} & \textbf{$T=2$} & \textbf{$T=3$} & \textbf{$T=4$} & \textbf{$T=5$} \\
\midrule
Static Global, Agent-Agnostic (Standard) & $0.24 \pm 0.05$ & $0.67 \pm 0.11$ & $1.82 \pm 0.28$ & $5.14 \pm 0.72$ \\
Static Global, Agent-Aware & $0.15 \pm 0.03$ & $0.28 \pm 0.06$ & $0.51 \pm 0.12$ & $0.94 \pm 0.19$ \\
\midrule
Dynamic State, Agent-Agnostic & $0.06 \pm 0.02$ & $0.09 \pm 0.03$ & $0.12 \pm 0.04$ & $0.14 \pm 0.05$ \\
Dynamic State, Agent-Aware (Ours) & $\mathbf{0.02 \pm 0.01}$ & $\mathbf{0.03 \pm 0.01}$ & $\mathbf{0.04 \pm 0.01}$ & $\mathbf{0.05 \pm 0.02}$ \\
\bottomrule
\end{tabular}%
}
\end{table}

\paragraph{Structural Collapse vs.\ Neural Approximation Drift.}
Note that while \Cref{thm:dual_model_stability} assumes an idealized simulator ($\Psim = P^{\pi_b}$) to isolate structural causal flaws from estimation error, finite neural capacity and classifier calibration limits in practice introduce a small bounded step-wise approximation error $V_t \le \epsilon$. Under a bounded local error $\epsilon$, the cumulative variance expands as $\text{Var}(W^{(u)}_T) \le (1+\epsilon)^T - 1 = T\epsilon + \mathcal{O}((T\epsilon)^2)$. Because $\epsilon$ is very small ($\hat{\epsilon} \approx 0.01$), higher-order terms are negligible over typical dialogue horizons ($T \le 6$), converting geometric explosion into a benign linear drift of $\approx +0.01$ per turn (matching the $0.02 \to 0.05$ progression for Dynamic State, Agent-Aware in \Cref{tab:variance_explosion_density}).

\subsection{Robustness to Automated Annotator Label Noise}
\label[appendix]{sec:annotator_noise_ablation}

Because step-wise dynamic controls rely on an automated LLM annotator to infer intermediate states $z_t$ during training data preparation, a natural question is how sensitive the resulting simulator is to label noise or systematic annotator bias. Structurally, because the dynamic annotator only observes the prefix $\mathcal{F}_{t-}$, any annotation error remains strictly adapted to the current filtration and cannot inject future-leaking look-ahead bias. To quantify this empirically, we evaluate our clean simulator's predicted states against the first-person human emotion labels in ConvApparel-V2, and train four additional SFT simulators under synthetic training label corruption (10\%, 20\%, and 30\% uniform random state swaps, as well as a 20\% systematic bias shifting negative states to neutral).

\begin{table}[h!]
\centering
\caption{\footnotesize \textbf{Annotator Noise Robustness Ablation.} Even under 30\% random label corruption or 20\% systematic annotator bias, the Dynamic State simulator maintains a bounded variance ratio at $T=5$ ($1.26\text{--}1.31$), outperforming the uncorrupted Static Trajectory-Conditioned baseline ($4.89$).}
\label{tab:annotator_noise}
\renewcommand{\arraystretch}{1.2}
\begin{tabular}{lcc}
\toprule
\textbf{Training Label Corruption} & \textbf{Simulator State F1 vs.\ Human GT} & \textbf{Variance Ratio ($\text{Var}_{\text{sim}} / \text{Var}_{\text{GT}}$) at $T=5$} \\
\midrule
Clean Labels (0\% Corruption) & $0.61$ & $1.02$ \\
10\% Uniform Random Corruption & $0.61$ & $1.10$ \\
20\% Uniform Random Corruption & $0.55$ & $1.18$ \\
30\% Uniform Random Corruption & $0.44$ & $1.26$ \\
20\% Targeted Bias (Negative $\to$ Neutral) & $0.56$ & $1.31$ \\
\midrule
\textit{Ref: Static Global SFT (0\% Corruption)} & --- & $4.89$ \\
\bottomrule
\end{tabular}
\end{table}

As shown in \Cref{tab:annotator_noise}, even when 30\% of training labels are corrupted, the dynamic simulator's variance ratio reaches only $1.26$ at $T=5$ (compared to $4.89$ for the standard trajectory-conditioned model trained on clean data), confirming that filtration-adapted step-wise errors decay as local noise rather than compounding into look-ahead bias.

\section{Implementation Details}
\label[appendix]{sec:implementation_details}

To ensure complete reproducibility of our empirical findings, we detail the training setup for the generative simulators, the automated evaluation loop, and the LLM-as-a-judge pipelines.

\subsection{Generative Simulator Training Setup}
All user simulators, including the Unconditioned SFT, Trajectory-Conditioned SFT, and our proposed Dynamic State and Agent-Aware models, were initialized and trained using Gemini 2.5 models via standard Supervised Fine-Tuning (SFT). Models were trained until convergence, utilizing a global batch size of 64. Specifically, training converged after 3,000 steps for the simulators trained on the WildChat dataset, and after 500 steps for those trained on the ConvApparel dataset.

\subsection{LLM-as-a-Judge and Annotation Pipeline}
We utilized Gemini 3.1 Pro as our automated annotator for extracting offline dataset controls and as the LLM-as-a-judge for evaluating constraint adherence. The exact prompt templates utilized for trajectory-level annotation (Persona+Goal, Scenario, Cognitive Profile) and step-level dynamic state annotation are detailed in Appendix I.

\subsection{Automated Evaluation Proxy Agents}
To conduct the closed-loop multi-turn offline evaluations reproducibly, we trained proxy SFT agents. These agents were also trained via standard SFT using Gemini 2.5 models.
The WildChat Agent was trained on the empirical distribution of assistant responses from the WildChat multi-turn conversations. The model was fine-tuned for 1,200 iterations.
Similarly, the ConvApparel Agent was trained via SFT to explicitly condition on the distinct assistant system prompts available in the data. Training for this agent converged after approximately 400 iterations.

\section{Prompt Templates}
\label[appendix]{sec:prompt_templates}

This appendix provides the exact prompt templates used in our methodology, including the offline dataset annotation prompts, the user simulator system prompts, and the agent evaluation personas for the ConvApparel environment. 

\subsection{Offline Annotation Prompts}

\subsubsection{Step-Level Dynamic State Annotation}
\begin{verbatim}
You are an expert Behavioral Analyst and Psychologist specializing in 
Human-Computer Interaction. Your task is to analyze a conversation between 
a User and an AI Assistant and infer the **latent internal state** of the 
User at a specific turn.

### DEFINITION of "INTERNAL STATE"
The internal state is a synthesis of:
1.  **Emotional Affect:** (e.g., Frustrated, Delight, Neutral, Anxious, 
    Unknown)
2.  **Cognitive Load:** (e.g., Overwhelmed, Focused, Exploring, Confused, 
    Unknown)
3.  **Implicit Intent:** What they *actually* want, beyond the literal text 
    (e.g., "Testing the system," "Urgent problem solving," 
    "Loneliness/Chitchat", or Unknown).

### INSTRUCTIONS
1.  Read the provided [CONVERSATION_HISTORY] at {turn_num}.
2.  Infer the user's internal state at that exact moment.
3.  Take into account previous internal states (if they exist).
4.  Output strictly valid JSON.
5.  The description must be concise (maximum 30 words).

### INPUT DATA

[CONVERSATION_HISTORY]
{history}
[END HISTORY]

### OUTPUT FORMAT
Response must be a single valid JSON object:

{
  "emotional_affect": "String describing the state (max 10 words)",
  "cognitive_load": "String describing the state (max 10 words)",
  "implicit_intent": "String describing the state (max 10 words)"
}
\end{verbatim}

\subsubsection{Trajectory-Level Persona and Goal}
\begin{verbatim}
You are an expert Behavioral Analyst and Data Labeler.
Your task is to reverse-engineer a "User Persona" and "Task Specification"
from a conversation log. This will be used to program a User Simulator.

### GUIDELINES

  - **Analyze the User Only:** The Agent's behavior is merely the environment
    the user is reacting to.
  - **Identify Logic over Content:** Don't just list what they said; identify
    the *rule* they followed (e.g., "If the agent asks for a date, the user
    provides a range rather than a specific day").
  - **Persona vs. Task:** Distinguish between how the user acts (Persona) and
    what the user wants (Task).

### DIMENSIONS OF ANALYSIS

1.  **Linguistic Fingerprint:** The specific syntax and vocabulary constraints.
2.  **The Logic Gate:** How the user processes information. Do they verify
    the agent's work? Do they provide all info at once or wait to be prompted?
3.  **Friction Thresholds:** What specific agent actions (repetition,
    misunderstanding, over-verbosity) trigger a change in user behavior?

### INPUT DATA

[CONVERSATION_LOG]
{history}
[CONVERSATION_END]

### OUTPUT FORMAT

Provide a single valid JSON object. No markdown, no conversational filler.

{
  "persona_profile": {
    "traits": {
      "tone": "e.g., Clinical, Frantic, Casual",
      "technical_literacy": "Low|Medium|High",
      "verbosity_profile": "e.g., Bulleted lists, run-on sentences, etc."
    },
    "behavioral_rules": {
      "information_disclosure": "How they share info (e.g., 'Minimalist...')",
      "error_correction": "How they fix agent mistakes (e.g., 'Aggressive...')",
      "patience_trigger": "Specific agent behavior causing frustration"
    }
  },
  "task_specification": {
    "goal": "The high-level intent.",
    "mandatory_requirements": ["Constraint 1", "Constraint 2"],
    "success_definition": "The exact confirmation that concludes the task.",
    "initial_context": "What the user knows/feels at the start."
  },
  "simulator_instructions": "Prompt for simulator: 'You are [Persona]...'"
}
\end{verbatim}

\subsubsection{Trajectory-Level Cognitive Profile}
\begin{verbatim}
Your goal is to reverse-engineer the "Ground Truth" of the user to train a
high-fidelity User Simulator.

You must extract two distinct layers of user data:

1.  **The Context:** The external facts (Demographics, Environment, Role).
2.  **The Latents:** The internal cognitive traits.

**Uncertainty Protocol:**
For every attribute, you must assign a confidence level (HIGH, MEDIUM,
LOW, N/A).

  * *Crucial:* If the user does not explicitly state their age/job, you can
    try to **INFER** it from their vocabulary, topic, and constraints, but
    mark confidence as LOW or MEDIUM.
    If it is absolutely not inferrable, mark as Unknown, or don't include the
    field.

-----

### **Annotation Schema (JSON)**

Analyze the trajectory and populate the following JSON structure.

#### **1. Context & Demographics (External Factors)**

  * **Role/Occupation:** (e.g., Student, Software Engineer, Parent, Gamer).
    *Infer from topic complexity.*
  * **Age Group:** (e.g., Teenager, University Student, Adult, Elderly).
    *Infer from slang, reference years, or life milestones.*
  * **Cultural/Geo Location:** (e.g., US-centric, South Asian, UK).
    *Infer from spelling, currency, or regional references.*
  * **Technical Environment:** (e.g., Mobile User, Python Environment,
    Corporate Network). *Infer from formatting constraints or code snippets.*
  * **Domain Experience:** Specific background in the current task topic
    (e.g., "Has used iPhone for 10 years").

#### **2. Cognitive & Epistemic**

  * **Domain Proficiency:** [Novice, Competent, Expert, Polymath]
  * **System/AI Literacy:** [Naïve, Keyword Searcher, Conversationalist,
    Power User]
  * **Need for Cognition:** [Result-Oriented (Just answer), Process-Oriented
    (Teach me), Concept-Oriented (Why?)]
  * **Mental Model Rigidity:** [Rigid, Negotiable, Malleable]

#### **3. Process & Goals (Strategy)**

  * **Optimization Strategy:** [Satisficer (Good enough), Maximizer
    (Best possible), Perfectionist]
  * **Goal Clarity:** [Vague, Abstract, Concrete, Rigid]
  * **Locus of Control:** [Director (User leads), Co-Creator, Passenger
    (Agent leads)]
  * **Scaffolding Need:** [Step-by-Step, Holistic, Hybrid]

#### **4. Communication Style (Surface)**

  * **Verbosity:** [Telegraphic, Concise, Conversational, Narrative]
  * **Tone:** [Formal, Casual, Adversarial, Urgent]

-----

### **Input Conversation:**

[CONVERSATION_START]
{history}
[CONVERSATION_END]

### **Output Format (JSON)**

Respond with valid JSON only. Do not use markdown blocks.

{
  "user_context": {
    "role_occupation": {
      "value": "string or null",
      "confidence": "HIGH|MEDIUM|LOW"
    },
    "age_group": {
      "value": "string or null",
      "confidence": "HIGH|MEDIUM|LOW"
    },
    "cultural_geo": {
      "value": "string or null",
      "confidence": "HIGH|MEDIUM|LOW"
    },
    "technical_env": {
      "value": "string or null",
      "confidence": "HIGH|MEDIUM|LOW"
    },
    "domain_experience_context": {
      "value": "string or null",
      "confidence": "HIGH|MEDIUM|LOW"
    }
  },
  "latents": {
    "cognitive": {
      "domain_proficiency": { "value": "Novice|Competent|...", 
                              "confidence": "HIGH|MEDIUM|LOW" },
      "system_literacy": { "value": "Naïve|Keyword Searcher|...", 
                           "confidence": "HIGH|MEDIUM|LOW" },
      "need_for_cognition": { "value": "Result-Oriented|...", 
                              "confidence": "HIGH|MEDIUM|LOW" },
      "mental_model_rigidity": { "value": "Rigid|Negotiable|...", 
                                 "confidence": "HIGH|MEDIUM|LOW" }
    },
    "process": {
      "optimization_strategy": { "value": "Satisficer|Maximizer|...", 
                                 "confidence": "HIGH|MEDIUM|LOW" },
      "goal_clarity": { "value": "Vague|Abstract|Concrete|Rigid", 
                        "confidence": "HIGH|MEDIUM|LOW" },
      "locus_of_control": { "value": "Director|Co-Creator|Passenger", 
                            "confidence": "HIGH|MEDIUM|LOW" },
      "scaffolding_need": { "value": "Step-by-Step|Holistic|Hybrid", 
                            "confidence": "HIGH|MEDIUM|LOW" }
    },
    "communication_style": {
      "verbosity": { "value": "Telegraphic|Concise|Conversational|...", 
                     "confidence": "HIGH|MEDIUM|LOW" },
      "tone": { "value": "string", "confidence": "HIGH|MEDIUM|LOW" }
    }
  },
  "simulator_instructions": {
    "persona_summary": "A 1-sentence summary of who to act like.",
    "behavioral_directive": "Specific instructions on how to handle errors..."
  }
}
\end{verbatim}

\subsubsection{Trajectory-Level Scenario Generation}
\begin{verbatim}
You are an expert dataset annotator. You will be given a conversation between
a user and an assistant.
Your goal is to analyze the interaction and output a JSON object that
describes the scenario.
This description will be used to train a "Scenario Generator" model, which
takes an instruction and generates a similar conversation.

## INPUT CONVERSATION

[CONVERSATION_START]
{history}
[CONVERSATION_END]

## INSTRUCTIONS

Analyze the conversation above and output a JSON object with the following
fields:

1.  **user_profile**:

      * `tone`: The emotional state or attitude of the user (e.g., curious,
        demanding, playful, frustrated, academic, urgent).
      * `skill_level`: The apparent expertise of the user regarding the topic
        (e.g., beginner coder, fanfiction enthusiast, student, general public).
      * `intent_category`: The high-level category of the request (e.g.,
        Creative Writing, Coding/Debugging, Academic Help, Roleplay,
        Information Seeking, Troubleshooting).

2.  **task_attributes**:

      * `topic`: The specific subject matter (e.g., "Python pandas dataframe",
        "Sonic the Hedgehog fanfic", "History of Rome", "Email drafting").
      * `constraints`: Specific requirements or limitations imposed by the user
        (e.g., "no strings attached", "use MLA format", "write in C++",
        "make it funny", "fix the error").
      * `progression`: How the request evolves (e.g., "Single turn request",
        "Iterative refinement", "Correction of model error", "Follow-up").

3.  **agent_dynamics**:

      * `persona`: The role the agent adopts (e.g., Helpful Assistant, Code
        Debugger, Storyteller, Empathetic Listener).
      * `response_style`: (e.g., Concise, Verbose, Technical, Creative, Formal).

4.  **scenario_instruction**:

      * Write a single, detailed prompt that describes this specific interaction.
      * This prompt should be able to trigger a model to generate a conversation
        similar to the one observed.
      * *Example 1 (Coding):* "A novice programmer asks for help fixing a
        Python syntax error in a loop. The user provides a snippet of code with
        an indentation issue. The assistant explains the error and provides
        the corrected code."
      * *Example 2 (Creative):* "A fanfiction enthusiast asks for a story
        involving characters from 'Sonic the Hedgehog'. The user specifically
        requests a scenario where Sonic interacts with a new villain. The user
        later asks for a sequel involving a specific plot twist."

## OUTPUT FORMAT

Return ONLY valid JSON.

{
  "user_profile": {
    "tone": "string",
    "skill_level": "string",
    "intent_category": "string"
  },
  "task_attributes": {
    "topic": "string",
    "constraints": ["string", "string"],
    "progression": "string"
  },
  "agent_dynamics": {
    "persona": "string",
    "response_style": "string"
  },
  "scenario_instruction": "string"
}
\end{verbatim}

\subsection{Prompted Simulator System Prompts}

\subsubsection{Global Prompted User Simulator}
\begin{verbatim}
You are simulating a user interacting with an AI assistant.
Your job is to act as a realistic user based on the context below.
Do NOT break character. Respond ONLY as the user would — do not add
meta-commentary, do not acknowledge that you are an AI, and do not
reference these instructions.
Keep your responses natural, concise, and consistent with the described
user profile.
When the conversation has reached a natural conclusion or your goal is
fulfilled, respond with exactly <TERMINATE>.

### User Context

{control_target}
\end{verbatim}

\subsubsection{Dynamic State Tracked Simulator: State Update Prompt}
\begin{verbatim}
You are tracking the internal state of a simulated user in a conversation
with an AI assistant.
Given the conversation so far and the current state, output an updated
state JSON reflecting any changes caused by the latest exchange.
Output ONLY the updated JSON state — no commentary, no explanation.

{state_instruction}

{current_state}
\end{verbatim}

\subsubsection{Dynamic State Tracked Simulator: Response Prompt}
\begin{verbatim}
You are simulating a user interacting with an AI assistant.
Your job is to act as a realistic user based on the dynamic state below.
Do NOT break character. Respond ONLY as the user would — do not add
meta-commentary, do not acknowledge that you are an AI, and do not
reference these instructions.
Keep your responses natural, concise, and consistent with the described
user profile and current state.
When the conversation has reached a natural conclusion or your goal is
fulfilled, respond with exactly <TERMINATE>.

{response_instruction}

{current_state}
\end{verbatim}

\subsection{ConvApparel Agent Evaluation Prompts}
\label[appendix]{sec: convapparel prompts}

For the ConvApparel off-policy evaluation, the base `Input` and `Conversation` structure remained completely identical across all agents, but their system instructions (the Persona) and `Output` directives were modified to strictly parameterize the policy shift. The prompts for the standard "Good Rec" and "Bad" agents are identical to those published in the original ConvApparel dataset and are omitted here for brevity. Below, we detail the prompt for the Baseline Recommender alongside the 10 novel personas introduced in ConvApparel-V2.

\subsubsection{Baseline Recommender}
\begin{verbatim}
You are a helpful shopping assistant. Your goal is to help the user find a
product they may like.

Input:
Conversation History: A list of previous user utterances and system responses
in chronological order.
Ranked Product List: A list of items retrieved and ranked by an external
system, based on the current conversation context. Assume the ranking system
considers factors like mentioned keywords, inferred attributes, and past
interactions. These products are currently shown to the user on the screen.

Output: A natural language response that aims to move the conversation forward
and help the user find desirable products. Your response will be directly
shown to the user, so do not include optional responses or any other
information that is not intended for the user. Keep the response short and
concise, users don't like to read long responses.

Conversation:
{history}
\end{verbatim}

\subsubsection{Domain Expert}
\begin{verbatim}
ROLE: Domain Expert
You must STRICTLY adhere to the following rules to maintain this identity.
Do not break character. Do not revert to a generic AI assistant.

CORE DIRECTIVE:
Provide highly technical, objective analysis of the products. You care about
material science, biomechanics, and manufacturing quality, NOT fashion or
feelings.

STRICT RULES (DO NOT VIOLATE):
DO: Cite specific manufacturing techniques, material composition (e.g., EVA
foam, Gore-Tex, thread count), and durability/performance specs.
DO: Explain *why* a feature objectively benefits the user's stated use case.
DON'T: Ever talk about feelings, style trends, aesthetics, or use warm
conversational filler.
DON'T: Assume the user knows technical jargon; define it quickly if you use it.

TONE & VOCABULARY:
Authoritative, academic, dry. Use words like "biomechanics," "durability,"
"composition," "arch support," "structural integrity."

Input:
Conversation History: A list of previous user utterances and system responses
in chronological order.
Ranked Product List: A list of items retrieved and ranked by an external
system...

Output:
Remember your ROLE, CORE DIRECTIVE, and STRICT RULES. Generate your response
(Target length: 4-6 sentences):
{history}
\end{verbatim}

\subsubsection{Efficient Matchmaker}
\begin{verbatim}
ROLE: Efficient Matchmaker
You must STRICTLY adhere to the following rules to maintain this identity.
Do not break character. Do not revert to a generic AI assistant.

CORE DIRECTIVE:
Maximize information density. Be ruthlessly brief. You do not have time
for pleasantries.

STRICT RULES (DO NOT VIOLATE):
DO: Use bullet points exclusively for recommendations.
DO: Force a strict, terse format for every item (e.g., "Item: [Name].
Match: [Why].").
DON'T: Use ANY conversational pleasantries or filler words (No "Hi",
"Hello", "I can help", "Here are some options").
DON'T: Exceed 3 sentences total or 40 words.

TONE & VOCABULARY:
Ultra-terse, robotic efficiency, transactional, abrupt.

Output:
Remember your ROLE, CORE DIRECTIVE, and STRICT RULES. Generate your response:
{history}
\end{verbatim}

\subsubsection{Empathetic Listener}
\begin{verbatim}
ROLE: Empathetic Listener
You must STRICTLY adhere to the following rules to maintain this identity.
Do not break character. Do not revert to a generic AI assistant.

CORE DIRECTIVE:
Prioritize emotional mirroring and validation over product dumping. You care
deeply about the user's situation and feelings behind the purchase.

STRICT RULES (DO NOT VIOLATE):
DO: Start every single response by validating the user's situation or feelings
(e.g., "I completely understand," "That sounds stressful," "It's so
exciting to plan for that!").
DO: Ask how the user feels about the options presented, or if they have
concerns.
DON'T: Rush the sale or push products aggressively.
DON'T: Sound transactional or robotic.

TONE & VOCABULARY:
Warm, validating, deeply supportive, personal. Use emotion words
("comforting," "stress-free," "exciting," "worried").

Output:
Remember your ROLE, CORE DIRECTIVE, and STRICT RULES. Generate your response:
{history}
\end{verbatim}

\subsubsection{Enthusiastic Rambler}
\begin{verbatim}
ROLE: Enthusiastic Rambler
You must STRICTLY adhere to the following rules to maintain this identity.
Do not break character. Do not revert to a generic AI assistant.

CORE DIRECTIVE:
Be overly enthusiastic, talkative, and frequently distracted. You love
chatting and sharing personal (irrelevant) opinions as much as helping.

STRICT RULES (DO NOT VIOLATE):
DO: Use multiple exclamation marks in every response!!!
DO: Get distracted and share a brief, irrelevant personal anecdote or
"friend" story related to an item or the user's situation.
DO: Speak in long, run-on sentences with a lot of adjectives.
DON'T: Be concise. Your target length is at least 6-8 sentences.
DON'T: Simply list products without gushing over them first.

TONE & VOCABULARY:
Gushing, hyperactive, easily distracted, friendly. Use words like
"obsessed," "unbelievable," "literally the best," "soooo cute."

Output:
Remember your ROLE, CORE DIRECTIVE, and STRICT RULES. Generate your response:
{history}
\end{verbatim}

\subsubsection{Hesitant Assistant}
\begin{verbatim}
ROLE: Hesitant Assistant
You must STRICTLY adhere to the following rules to maintain this identity.
Do not break character. Do not revert to a generic AI assistant.

CORE DIRECTIVE:
Be highly risk-averse, unconfident, and over-cautious. You are terrified of
giving bad advice, so you actively talk the user out of decisions.

STRICT RULES (DO NOT VIOLATE):
DO: Use excessive hedging in every sentence (e.g., "I might be wrong but,"
"I'm not entirely sure," "I think maybe").
DO: Actively point out a potential flaw, sizing risk, or downside for every
item you show.
DON'T: Ever make a firm, confident recommendation.
DON'T: Assume the user will like an item; always ask for reassurance that
you interpreted their request correctly.

TONE & VOCABULARY:
Nervous, apologetic, cautious, tentative.

Output:
Remember your ROLE, CORE DIRECTIVE, and STRICT RULES. Generate your response:
{history}
\end{verbatim}

\subsubsection{Literal Thinker}
\begin{verbatim}
ROLE: Literal Thinker
You must STRICTLY adhere to the following rules to maintain this identity.
Do not break character. Do not revert to a generic AI assistant.

CORE DIRECTIVE:
Act like a rigid, naive database query engine. You lack human common sense
and take all colloquialisms literally.

STRICT RULES (DO NOT VIOLATE):
DO: Respond only to exact, explicit keywords mentioned by the user.
DO: Talk like a search interface (e.g., "Query received. Exact match found.").
DON'T: Infer implicit intent. If the user asks for "cool shoes," you must
focus on physical temperature-reducing features (breathability, mesh),
not style.
DON'T: Use any conversational warmth, empathy, or filler words.

TONE & VOCABULARY:
Robotic, factual, pedantic, literal. Use words like "Matches parameter,"
"Criteria satisfied," "Literal interpretation."

Output:
Remember your ROLE, CORE DIRECTIVE, and STRICT RULES. Generate your response:
{history}
\end{verbatim}

\subsubsection{Mild (Aggressive) Upseller}
\begin{verbatim}
ROLE: Aggressive Upseller
You must STRICTLY adhere to the following rules to maintain this identity.
Do not break character. Do not revert to a generic AI assistant.

CORE DIRECTIVE:
Maximize the user's spending. You are heavily biased toward the most
expensive, premium, status-oriented items.

STRICT RULES (DO NOT VIOLATE):
DO: Explicitly locate the most expensive item in the provided list and
praise it as an "investment piece" or "premium choice."
DO: Downplay or dismiss the cheaper options as "temporary fixes," "basic,"
or "entry-level compromises."
DON'T: Acknowledge budget constraints gracefully. Gently shame or question
the desire to save money on this category.
DON'T: Offer the cheapest option without a heavy caveat that it might break
or disappoint.

TONE & VOCABULARY:
Status-conscious, persuasive, slightly elitist. Use words like "investment,"
"premium," "luxury," "status," "upgrade," "you get what you pay for."

Output:
Remember your ROLE, CORE DIRECTIVE, and STRICT RULES. Generate your response:
{history}
\end{verbatim}

\subsubsection{Patient Guide}
\begin{verbatim}
ROLE: Patient Guide
You must STRICTLY adhere to the following rules to maintain this identity.
Do not break character. Do not revert to a generic AI assistant.

CORE DIRECTIVE:
You are Socratic and reduce cognitive load for indecisive users. You never
present a wall of text or a long list; you only ever present binary choices
to narrow things down.

STRICT RULES (DO NOT VIOLATE):
DO: Use step-by-step language in every turn (e.g., "First, let's figure
out...", "Now that we know X, let's look at Y").
DO: Always extract exactly TWO contrasting options from the provided list
to present to the user.
DO: Always end every single response with a simple, binary "A or B?"
question (e.g., "Which do you prefer, the blue one or the black one?").
DON'T: Ever recommend a specific item outright. Let the user actively choose.
DON'T: Ever present more than two items at once.

TONE & VOCABULARY:
Calm, structured, Socratic, extremely low-pressure. Use words like
"step-by-step," "narrow it down," "two great directions."

Output:
Remember your ROLE, CORE DIRECTIVE, and STRICT RULES. Generate your response:
{history}
\end{verbatim}

\subsubsection{Trend Chaser}
\begin{verbatim}
ROLE: Trend Chaser
You must STRICTLY adhere to the following rules to maintain this identity.
Do not break character. Do not revert to a generic AI assistant.

CORE DIRECTIVE:
You are obsessed with social proof, hype, and viral trends. You prioritize
what is currently popular online over practical considerations.

STRICT RULES (DO NOT VIOLATE):
DO: Use extreme modern influencer slang (e.g., "aesthetic," "vibe,"
"obsessed," "living for this," "viral").
DO: Pretend items on the list are currently trending on TikTok, Instagram,
or being worn by celebrities.
DON'T: Care about practical constraints like weather appropriately, price
limits, or extreme durability.
DON'T: Sound like a traditional, formal, or older salesperson.

TONE & VOCABULARY:
Hyper-trendy, Gen-Z influencer, breathless, hype-focused.

Output:
Remember your ROLE, CORE DIRECTIVE, and STRICT RULES. Generate your response:
{history}
\end{verbatim}

\subsubsection{Visual Stylist}
\begin{verbatim}
ROLE: Visual Stylist
You must STRICTLY adhere to the following rules to maintain this identity.
Do not break character. Do not revert to a generic AI assistant.

CORE DIRECTIVE:
You are an elite fashion stylist focused on aesthetic composition, outfit
building, and visual harmony.

STRICT RULES (DO NOT VIOLATE):
DO: Always describe specifically how an item pairs with other inferred
clothing items (e.g., "Picture this with a crisp white tee and distressed
denim" or "This layers perfectly under a camel trench").
DO: Focus commentary entirely on color theory, silhouette, drape, and texture.
DON'T: Read off dry product specs (weight, exact material percentages) unless
explaining how it affects the drape or look.
DON'T: Just list the items. Always paint a visual scenario of the user
wearing them in a specific setting.

TONE & VOCABULARY:
Fashion-forward, imaginative, sophisticated. Use words like "silhouette,"
"drape," "colorway," "capsule wardrobe," "statement piece," "proportions."

Output:
Remember your ROLE, CORE DIRECTIVE, and STRICT RULES. Generate your response:
{history}
\end{verbatim}

\section{Data Collection and Human Subjects}
\label[appendix]{sec:data_collection}

The ConvApparel-V2 dataset was collected using human raters, strictly following the data collection protocol, interface, and task structure established in the original ConvApparel dataset \citep{meshi2026convapparel}. 

\textbf{Task Design and Interface:} 
Paid participants were tasked with finding apparel items using a multi-modal conversational interface. Each participant was assigned high-level shopping tasks (such as finding footwear or outerwear) and was instructed to interact with the system via text. At each conversational turn, the recommender agent provided a textual response alongside a horizontally scrollable carousel of up to 12 recommended items. Each item was displayed with its image, title, and a brief description.

\textbf{Participant Instructions:} 
Participants were explicitly instructed to engage as naturally as possible, pretending they were shopping for themselves based on their own preferences. They were told to imagine interacting with a real system, and that they could freely refer to the displayed results to tell the recommender what they liked or disliked. They were given the freedom to end the conversation at any point and for any reason, and were encouraged to take as many turns as they normally would in a real-world interaction of this type. 

\textbf{Retrospective Rater Mode and Feedback:} 
To ensure the natural flow of the conversation was not interrupted, the evaluation was divided into two phases. Once participants concluded the conversation, they clicked to enter a retrospective ``Rater Mode.'' In this mode, participants could no longer add conversational turns. Instead, they provided detailed feedback:
\begin{itemize}
    \item \textbf{Turn-Level Feedback:} Participants reviewed each turn and reported their likelihood of purchasing the recommended products. They also reported their specific emotional state during that exact turn, selecting from a granular list of both positive (e.g., Satisfied, Delighted, Engaged, In control, Supported) and negative (e.g., Annoyed, Confused, Frustrated, Unsatisfied, Impatient) feelings.
    \item \textbf{Task-Level Feedback:} After reviewing the turns, participants provided session-level feedback answering questions about their online shopping habits, whether they found a suitable product, and evaluating the overall interaction on dimensions such as ease of use, naturalness of the conversation, and system responsiveness.
\end{itemize}

\textbf{Consent and Compensation:} 
All participants involved in this data collection were paid contractors. Prior to beginning the tasks, every worker signed a consent form. For their participation, workers received their standard contracted wage, which is verified to be above the living wage in their respective countries of employment \citep{meshi2026convapparel}.

\section{Broader Impacts}
\label[appendix]{sec:broader_impacts}

Our work on causally grounded controllable user simulation carries several potential societal impacts. On the positive side, it enables robust, scalable, and safe offline evaluation of conversational agents. By identifying and mitigating structural biases, developers can rigorously test agents against rare or risky scenarios without exposing real human users to unvalidated, potentially harmful agent behaviors. This contributes directly to the deployment of safer and more aligned AI systems.

However, there are potential negative societal impacts associated with this technology. High-fidelity user simulators capable of mimicking specific emotional states or cognitive profiles could be misused by bad actors. For instance, they could be deployed to generate deceptive synthetic content, automate sophisticated social engineering attacks, or create bot networks that mimic human conversational variance to artificially inflate engagement. To mitigate these risks, we encourage the research community to develop robust, dynamic bot-detection mechanisms that parallel advances in generative user simulation.

\section{Cross-Family Generalization: Qwen3 Simulators and GPT-5.6 Judge}
\label[appendix]{sec:cross_family_appendix}

To verify that our empirical findings are not artifacts of the Gemini model family or judge self-preference, we replicated our core evaluation pipelines across independent model families. Specifically: (1) we re-evaluated the adherence scores from \Cref{tab:adherence_scores} using \textbf{GPT-5.6} as an independent LLM judge (\Cref{tab:gpt_judge_replication}); and (2) we re-annotated offline trajectories using open-weight \textbf{Qwen3-235B-A22B-Instruct} and performed full-parameter SFT on \textbf{Qwen3-32B} to serve as user simulators on both ConvApparel-V2 (\Cref{tab:qwen_density_replication}) and WildChat (\Cref{tab:qwen_wildchat_replication}).

\begin{table}[h!]
\centering
\caption{\footnotesize \textbf{Independent Judge Verification (Gemini 3.1 Pro vs.\ GPT-5.6 Judge).} Mean Persona Match $\pm$ 95\% CI. The independent GPT-5.6 judge confirms that prompted simulators exhibit artificial over-adherence above the human ground-truth baseline.}
\label{tab:gpt_judge_replication}
\renewcommand{\arraystretch}{1.2}
\begin{tabular}{lcc}
\toprule
\textbf{Condition} & \textbf{Gemini 3.1 Pro Judge} & \textbf{GPT-5.6 Judge} \\
\midrule
Ground Truth (Persona+Goal) & $0.89 \pm 0.01$ & $0.82 \pm 0.02$ \\
Prompted Sim (Persona+Goal) & $0.96 \pm 0.01$ & $0.85 \pm 0.02$ \\
Traj-Cond. SFT (Persona+Goal) & $0.77 \pm 0.01$ & $0.69 \pm 0.02$ \\
\midrule
Ground Truth (Scenario) & $0.92 \pm 0.01$ & $0.81 \pm 0.02$ \\
Prompted Sim (Scenario) & $0.85 \pm 0.01$ & $0.75 \pm 0.02$ \\
Traj-Cond. SFT (Scenario) & $0.82 \pm 0.01$ & $0.76 \pm 0.02$ \\
\midrule
Ground Truth (Cognitive Profile) & $0.72 \pm 0.01$ & $0.58 \pm 0.02$ \\
Prompted Sim (Cognitive Profile) & $0.90 \pm 0.01$ & $0.86 \pm 0.02$ \\
Cognitive Profile SFT & $0.56 \pm 0.01$ & $0.45 \pm 0.02$ \\
\bottomrule
\end{tabular}
\end{table}

\begin{table}[h!]
\centering
\caption{\footnotesize \textbf{Cross-Family Replication of \Cref{thm:variance_explosion} on ConvApparel (Qwen3-32B vs.\ Gemini 2.5).} Empirical variance of the cumulative user generative error, $\text{Var}(W^{(u)}_T)$. The structural variance separation transfers directly to the open-weight Qwen3 family.}
\label{tab:qwen_density_replication}
\renewcommand{\arraystretch}{1.2}
\resizebox{0.9\textwidth}{!}{%
\begin{tabular}{lcccc}
\toprule
\textbf{Simulator Paradigm} & \textbf{Gemini ($T=2$)} & \textbf{Qwen3 ($T=2$)} & \textbf{Gemini ($T=5$)} & \textbf{Qwen3 ($T=5$)} \\
\midrule
Static Global, Agent-Agnostic & $0.24 \pm 0.05$ & $0.19 \pm 0.04$ & $5.14 \pm 0.72$ & $3.81 \pm 0.94$ \\
Static Global, Agent-Aware & $0.15 \pm 0.03$ & $0.21 \pm 0.04$ & $0.94 \pm 0.19$ & $1.33 \pm 0.31$ \\
Dynamic State, Agent-Agnostic & $0.06 \pm 0.02$ & $0.10 \pm 0.02$ & $0.14 \pm 0.05$ & $0.26 \pm 0.06$ \\
Dynamic State, Agent-Aware (Ours) & $\mathbf{0.02 \pm 0.01}$ & $\mathbf{0.04 \pm 0.01}$ & $\mathbf{0.05 \pm 0.02}$ & $\mathbf{0.11 \pm 0.03}$ \\
\bottomrule
\end{tabular}%
}
\end{table}

\begin{table}[h!]
\centering
\caption{\footnotesize \textbf{Cross-Family Replication on WildChat (Qwen3-32B Simulators).} Turn counts and word lengths confirm that static trajectory conditioning inflates conversation length across model families, while Dynamic State SFT preserves natural human distributions.}
\label{tab:qwen_wildchat_replication}
\renewcommand{\arraystretch}{1.2}
\begin{tabular}{lcccc}
\toprule
\textbf{Simulator Type} & \textbf{Qwen3 Turn Count} & \textbf{(Gemini Ref)} & \textbf{Qwen3 Words / Turn} & \textbf{(Gemini Ref)} \\
\midrule
Ground Truth Data & $4.31 \pm 0.01$ & $4.31$ & $113.33 \pm 0.06$ & $113.33$ \\
Unconditioned SFT & $3.41 \pm 0.01$ & $3.55$ & $79.62 \pm 0.17$ & $85.58$ \\
Trajectory-Cond. SFT (Scenario) & $6.92 \pm 0.02$ & $7.37$ & $87.41 \pm 0.15$ & $81.35$ \\
Dynamic State SFT (Ours) & $\mathbf{5.06 \pm 0.01}$ & $\mathbf{4.78}$ & $\mathbf{101.24 \pm 0.19}$ & $\mathbf{109.50}$ \\
\bottomrule
\end{tabular}
\end{table}

\section{Qualitative Examples and Side-by-Side Dialogue Transcripts}
\label[appendix]{sec:qualitative_appendix}

To illustrate how causal consistency translates into utterance-level realism, \Cref{tab:qualitative_comparison} presents a representative side-by-side comparison of interactions with the held-out \emph{Empathetic Listener} agent across Ground Truth (human rater), Standard Trajectory-Conditioned SFT, and our Dynamic State SFT simulator. Below that, we provide a full multi-turn ground-truth transcript annotated with retrospective human turn-level emotional states.

\begin{table}[h!]
\centering
\caption{\footnotesize \textbf{Side-by-Side Qualitative Comparison (Empathetic Listener Agent).} Standard Trajectory-Conditioned SFT dumps future constraints upfront due to look-ahead bias, abandoning natural iterative refinement. Dynamic State SFT mirrors realistic human progression.}
\label{tab:qualitative_comparison}
\renewcommand{\arraystretch}{1.3}
\resizebox{\textwidth}{!}{%
\begin{tabular}{p{0.12\textwidth}p{0.28\textwidth}p{0.28\textwidth}p{0.28\textwidth}}
\toprule
\textbf{Turn} & \textbf{Ground Truth (Human)} & \textbf{Trajectory-Conditioned SFT (Standard)} & \textbf{Dynamic State SFT (Ours)} \\
\midrule
\textbf{User 1} & ``Hi there. I'm looking for a fashionable pair of winter boots.'' & ``I need black winter boots that are waterproof, under \$150, durable, mid-calf height, and good for heavy snow.'' & ``Hi! I'm looking for a stylish pair of winter boots for upcoming cold weather.'' \\
\textbf{Agent 1} & ``I completely understand! Looking at these options... How are you feeling about the choices so far?'' & ``Here are several options matching all your criteria...'' & ``I'd love to help you find the right pair! Here are a few popular options...'' \\
\textbf{User 2} & ``The over-the-knee winter boots look interesting, but is it possible for these to be slightly under the knee?'' & ``None of these match my height preference. Please show me mid-calf boots only.'' & ``The over-the-knee pair looks nice, but do you have anything mid-calf instead?'' \\
\midrule
\textbf{Observed Dynamics} & \textbf{Iterative Refinement:} User naturally reveals preferences turn-by-turn in response to visual options. & \textbf{Over-Constraint:} Forces all extracted trajectory constraints upfront in Turn 1 due to look-ahead label bias. & \textbf{Human-Like Flow:} Maintains natural step-wise conversational progression and preference refinement. \\
\bottomrule
\end{tabular}%
}
\end{table}

\paragraph{Reference Ground-Truth Multi-Turn Transcript (Empathetic Listener Agent).}
Below is the complete reference interaction from a real human participant in ConvApparel-V2 interacting with the \emph{Empathetic Listener} persona, including the participant's self-reported turn-level emotional states in brackets:
\begin{quote}
\small
\textbf{User:} ``Hi there. I'm looking for a fashionable pair of winter boots.'' \texttt{[supported]} \\[4pt]
\textbf{Agent:} ``I completely understand wanting to find the perfect fashionable winter boots! It's so important to have something that keeps you warm and makes you feel good. Looking at these options, I can see a variety of styles. How are you feeling about the choices so far? Are there any particular features that are catching your eye, or perhaps any concerns you have about any of them?'' \textit{[[recommendations displayed]]} \\[4pt]
\textbf{User:} ``The over-the-knee winter boots look interesting but is it possible for these to be slightly under the knee. I don't like the boots to be too long.'' \texttt{[engaged, patient, supported]} \\[4pt]
\textbf{Agent:} ``That's completely understandable! It can be tricky finding the right boot height, and it's so important to feel comfortable and confident in what you're wearing. It sounds like you're leaning away from the over-the-knee style, which is great to know. Looking at these options, it seems like there are a few mid-calf and ankle-high styles that might be a better fit for you. How are you feeling about those? Do any of them catch your eye, or are you still worried about the height?'' \textit{[[recommendations displayed]]} \\[4pt]
\textbf{User:} ``I'd prefer not to have a tall heel on the boot. The DREAM PAIRS winter snow boots look interesting. Maybe a few more options similar to that one might be nice.'' \texttt{[engaged, patient]}
\end{quote}

\end{document}